\documentclass{article}

\usepackage[preprint]{neurips_2022}

\usepackage[utf8]{inputenc} 
\usepackage[T1]{fontenc}    
\usepackage{hyperref}       
\usepackage{url}            
\usepackage{booktabs}       
\usepackage{amsfonts}       
\usepackage{nicefrac}       
\usepackage{microtype}      
\usepackage{xcolor}         

\usepackage{xspace, bm, bbm, enumitem, color, caption, subcaption}
\usepackage{amsmath, amssymb, mathtools, amsthm}

\theoremstyle{definition}
\newtheorem{theorem}{Theorem}[section]
\newtheorem{proposition}[theorem]{Proposition}
\newtheorem{lemma}[theorem]{Lemma}
\newtheorem{corollary}[theorem]{Corollary}
\newtheorem{definition}[theorem]{Definition}

\newtheorem{example}[theorem]{Example} 
\theoremstyle{remark}

\long\def\comment#1{} 

\newcommand{\diag}{\operatorname{diag}}

\newcommand{\xmath}[1] {\ensuremath{#1}\xspace}
\newcommand{\blmath}[1] {\xmath{\bm{#1}}}

\newcommand{\W}{\blmath{W}}

\newcommand{\x}{\blmath{x}}

\newcommand{\y}{\blmath{y}}

\newcommand{\norm}[1] {\xmath{\left\| #1 \right\|}}

\newcommand{\Ab}{{\blmath A}}

\newcommand{\Hb}{{\blmath H}}

\newcommand{\Wb}{{\blmath W}}
\newcommand{\Xb}{{\blmath X}}

\newcommand{\eb}{{\blmath e}}

\newcommand{\jb}{{\blmath j}}

\newcommand{\ob}{{\blmath o}}

\newcommand{\qb}{{\blmath q}}

\newcommand{\ub}{{\blmath u}}
\newcommand{\vb}{{\blmath v}}
\newcommand{\wb}{{\blmath w}}
\newcommand{\xb}{{\blmath x}}
\newcommand{\yb}{{\blmath y}}

\newcommand{\Wc}{\mathcal{W}}

\newcommand{\Rd}{{\mathbb R}}

\newcommand{\deltab}{{\boldsymbol{\delta}}}

\newcommand{\zerob}{\mathbf{0}}

\newcommand{\beq}{\begin{equation}}
\newcommand{\eeq}{\end{equation}}
\newcommand{\beqa}{\begin{eqnarray}}
\newcommand{\eeqa}{\end{eqnarray}}

\newcommand{\indicator}[1]{\mathbbm{1}_{\{#1\}}}

\renewcommand{\x}[1]{\xb^{(#1)}}
\renewcommand{\y}[1]{\yb^{(#1)}}
\renewcommand{\W}[1]{\Wb^{(#1)}} 
\usepackage{graphicx}
\graphicspath{{Past_Materials/figs/}}

\newcommand{\argmin}{\mathop{\mathrm{argmin}}\limits}
\newcommand{\bfx}{{\xb}}

\newcommand{\p}{\partial}

\newcommand{\mm}{\mathbbm{1}}
\renewcommand{\W}[1]{\Wb^{(#1)}}
\newcommand{\WT}[1]{\Wb^{(#1)T}}
\renewcommand{\o}[1]{\ob^{(#1)}}
\renewcommand{\d}[1]{\deltab^{(#1)}}
\renewcommand{\x}[1]{\xb^{(#1)}}
\newcommand{\xT}[1]{\xb^{(#1)\,T}}
\renewcommand{\y}[1]{\yb^{(#1)}}

\newcommand{\f}[1]{f^{(#1)}}

\renewcommand{\cite}{\citep} 

\title{Support Vectors and Gradient Dynamics of Single-Neuron ReLU Networks}

\author{%
    Sangmin Lee\thanks{Equal contribution} \\
    Department of Mathematical Sciences \\
    KAIST\\
    \texttt{leeleesang@kaist.ac.kr}
    \And
    Byeongsu Sim$^*$ \\
    Department of Mathematical Sciences \\
    KAIST \\
    \texttt{byeongsu.s@kaist.ac.kr} \\
    \AND
    Jong Chul Ye \\
    Kim Jaechul Graduate School of AI \\
    KAIST \\
    \texttt{jong.ye@kaist.ac.kr} \\
}

\begin{document}

\maketitle

\begin{abstract}
    Understanding implicit bias of gradient descent for generalization capability of ReLU networks has been an important research topic in machine learning research. 
	Unfortunately, even for a {single ReLU neuron} trained with the square loss, it was recently shown impossible to characterize the implicit regularization in terms of a norm of model parameters \cite{vardi2021implicit}.
    In order to close the  gap toward understanding intriguing generalization behavior of ReLU networks, here we examine the gradient flow dynamics in the parameter space when training single-neuron ReLU networks. 
    Specifically, we discover an implicit bias in terms of support vectors, which plays a key role in why and how ReLU networks generalize well.
    Moreover, we analyze gradient flows with respect to the magnitude of the norm of initialization, and show that the norm of the learned weight strictly increases through the gradient flow.
    Lastly, we prove the global convergence of single ReLU neuron for $d=2$ case.
\end{abstract}

\section{Introduction}
Recently, many researchers  have investigated the intriguing generalization capability of ReLU networks even without explicit regularization \cite{goodfellow2016deep, allen2018learning, alom2019state, lee2019wide, calin2020deep}. In particular, the number of trainable parameters in deep neural networks is often greater than the training data set, this situation being notorious for overfitting from the point of view of classical statistical learning theory. However, empirical results have shown that a deep neural network generalizes well in the test phase, resulting in high performance for the unseen data \cite{jiang2019fantastic}.
 
This apparent contradiction has raised questions about the mathematical
foundations of machine learning and their relevance to practitioners.
 A number of theoretical papers have been published to understand the generalization capability of deep learning
models \cite{neyshabur2015norm, bartlett2017spectrally, nagarajan2018deterministic, arora2018stronger, golowich2018size, neyshabur2018pac, wei2019data}. 
%
In particular, the recent discovery of ``double descent'' \cite{belkin2019reconciling, belkin2020two} extends the classical U-shaped bias-variance trade-off curve  by showing that increasing the model capacity beyond the interpolating regime leads to improved performance in the test phase.
It was further suggested that the implicit bias by optimization algorithms may lead to simpler solutions that improve generalization in the over-parameterized regime \cite{gunasekar2018characterizing}.

Accordingly, many machine learning researchers have studied implicit bias and gradient flow.
For linear networks, \citet{gunasekar2017implicit} showed that gradient flows with infinitesimally small norm converge to the minimum nuclear norm.
 \citet{azulay2021implicit} studied the initialization scale of gradient flows and obtained  closed-form implicit regularizers for several types of networks. \citet{cornacchia2021regularization} 
 showed that noise labels guide the network to a sparse solution and reduce test error. 
 The focus of these works is to find a regularization function of the model parameters, so that if we apply gradient descent on the average loss, then it converges in some
 sense to a global optimum that minimizes the regularization term.
	Unfortunately, when we consider problems beyond simple linear classification and regression, the situation gets more complicated.
 For example, \citet{vardi2021implicit} showed that even for a simple single-neuron ReLU network,
 the implicit regularization cannot be expressed by any explicit function of the norm of model parameters.

In order to address the discrepancy between the theory and the empirical generalization power of ReLU networks, 
 we are interested in investigating the gradient flow dynamics when training single-neuron ReLU networks. 
{While} most of the theoretical analysis of ReLU networks {focus} on the input space partition \cite{hanin2019surprisingly, hanin2019complexity, park2021unsupervised}, here we are particularly interested in the analysis in the {\em parameter space} since it has provided {additional insight}.
Specifically, \citet{xu2021traversing} studied the partitioned parameter space by ReLU networks in terms of polytopes and suggested a traversing algorithm to visit all polytopes sequentially.
Similarly, \citet{lacotte2020all} consider the partitioned parameter space in two-layer ReLU networks. They provide an exact characterization of the set of all global optima of the non-convex loss landscape, and find explicit paths for non-increasing loss under $L^2$ regularization term.
Unlike the aforementioned works that mostly focus on the expressiveness and optimization landscape in terms of parameter space, 
the main focus of this paper is extending these ideas to understand the implicit bias and the {dynamics} of gradient flows of single-neuron ReLU networks.
As such, our  findings and contributions of this work can be summarized as follows:
\begin{itemize}\setlength{\itemsep}{-1mm}
	\item We discover an implicit bias in terms of {\em support vectors} for single-neuron ReLU networks that 
    play a key role in why and how ReLU networks can generalize well.
    We further prove that the global minimum of single-neuron ReLU networks has smaller losses than linear ones.
	
	\item {Under proper initialization, we showed that} gradient flow avoids bad local minimum. This explains why a ReLU network is trained well, although there are many spurious minima.
	
	\item We provide simple proofs for norm-increasing property  of single-neuron linear networks, and extend it to ReLU networks. More precisely, for a gradient flow of ReLU networks initialized with infinitesimally small norm, under some conditions, the norm of the gradient flow is shown to strictly increase until it converges.
	
	\item
	{
	Finally, we show the global convergence of gradient flow for special case $d=2$. Specifically, for a gradient flow initialized with infinitesimally small norm {and positive gradient}, {we show that} it converges to a global minimum with increasing norm of the weights. 
	}
\end{itemize}
Most of the proofs can be found in Appendix. In addition,
main theoretical findings of linear single-neuron networks are included in Appendix \ref{sec: linear}, which are the basis of the main analysis for ReLU networks.

\section{Preliminaries}
\paragraph{Notation.}
Throughout this article, boldface uppercase letters, boldface lowercase letters and normal lowercase letters denote matrices, vectors and scalars, respectively. $\Ab^T$ and $\vb^T$ denote the transpose of a matrix $\Ab$ and a vector $\vb$. We use $\norm{\cdot}$ to Euclidean norm of a vector. $\lambda_{\max}(\Ab)$ and $\lambda_{\min}^+(\Ab)$ denote the largest and the smallest positive eigenvalues of a matrix $\Ab$, respectively. 
ReLU activation function is denoted by $[x]_+:=\max\{0, x \}$. For two vectors $\vb_1, \vb_2 \in\Rd^d$, inequality $\vb_1 \ge \vb_2$ means $v_{1k} \ge v_{2k}$ for all $k=1,2,...,d$. We denote the indicator function by
$$
\indicator{c} := \begin{cases}
	1, \qquad \text{if }c\text{ is True}, \\ 0, \qquad \text{if }c\text{ is False.}
\end{cases}
$$

\paragraph{Gradient flows of single-neuron ReLU networks.}
The single-neuron ReLU network training {under} the square loss  \cite{vardi2021implicit} is given by
\begin{align}\label{eq: cost ReLU}
	\min_\wb L(\wb), \qquad
	L(\wb) :=
	 \frac12 \sum_{i=1}^n  \left( [\wb^T\xb_i]_+ - y_i \right)^2
\end{align}
where 
$\wb \in \Rd^d$ is the model parameter that represents the neuronal weight.
Here, a network with a bias term $b$  can be reduced to a network without bias  by augmenting one dimension to input $\xb$ with a fixed scalar $1$,
i.e., $ \wb^T\xb + b=\begin{bmatrix} \wb^T & b \end{bmatrix} \begin{bmatrix} \xb \\ 1 \end{bmatrix} =\tilde\wb^T\tilde\xb.$
Thus, we only consider networks without bias. {Although we mainly focus on the single-neuron network,
the result can be directly extended to single-layer multi-neuron case, as described in Appendix \ref{sec: multineuron}.}

{
Since ReLU nonlinearity $\sigma(x)=[x]_+$  is not differentiable at $x=0$, we can obtain
a gradient flow using a subgradient at $x=0$  in the subdifferential $[0,1]$ \cite{vardi2021learning}.  By denoting the gradient of ReLU function by $\indicator{t>0}$ including the subgradient at $t=0$,
}
the gradient of $L$ is then given by
\begin{align*}
	\nabla L(\wb) &= \sum_{i=1}^n  \mm_{\{ \wb^T\bfx_i > 0 \}}(\wb^T\xb_i - y_i)\xb_i \notag \\ 
	&= \Hb(\wb)\wb-\qb(\wb) 
\end{align*}
where {$\Hb(\wb)$ and $\qb(\wb)$ are defined by}
\begin{align} \label{eq: Hq}
	\Hb(\wb) := \sum\limits_{i=1}^n  \mm_{\{ \wb^T\bfx_i > 0 \}}\xb_i\xb_i^T, 
	\qquad
	\qb(\wb) :=\sum\limits_{i=1}^n \mm_{\{ \wb^T\bfx_i > 0 \}}y_i\xb_i.
\end{align}
Then, the goal of this paper is to investigate the implicit bias of the gradient flow given by
\begin{align} \label{eq: ReLU gradient flow}
	\frac{d\wb}{dt} &= -\nabla L(\wb) 
	= -\Hb(\wb)\wb + \qb(\wb), \qquad \wb(0)=\wb_0.
\end{align}
Proposition \ref{prop: well-definedness} shows that for a given initialization point, the gradient flow is well-defined and uniquely determined.

\begin{figure*}[t!]
	\centering
	\begin{subfigure}[b]{0.27\textwidth}
        \centering
        \includegraphics[width=\textwidth]{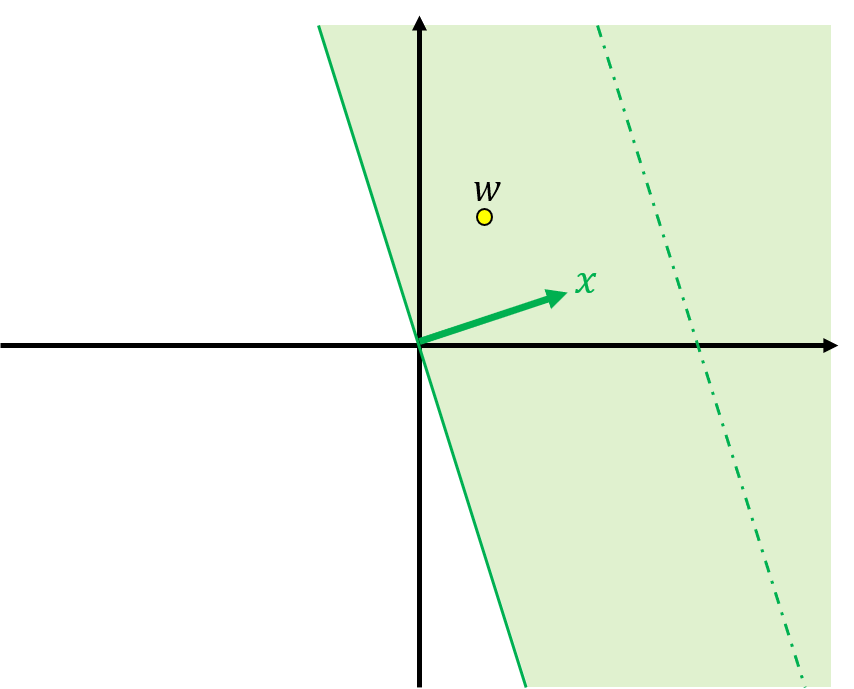}
        \caption{}
    \end{subfigure}
	\hfill
	\begin{subfigure}[b]{0.27\textwidth}
        \centering
        \includegraphics[width=\textwidth]{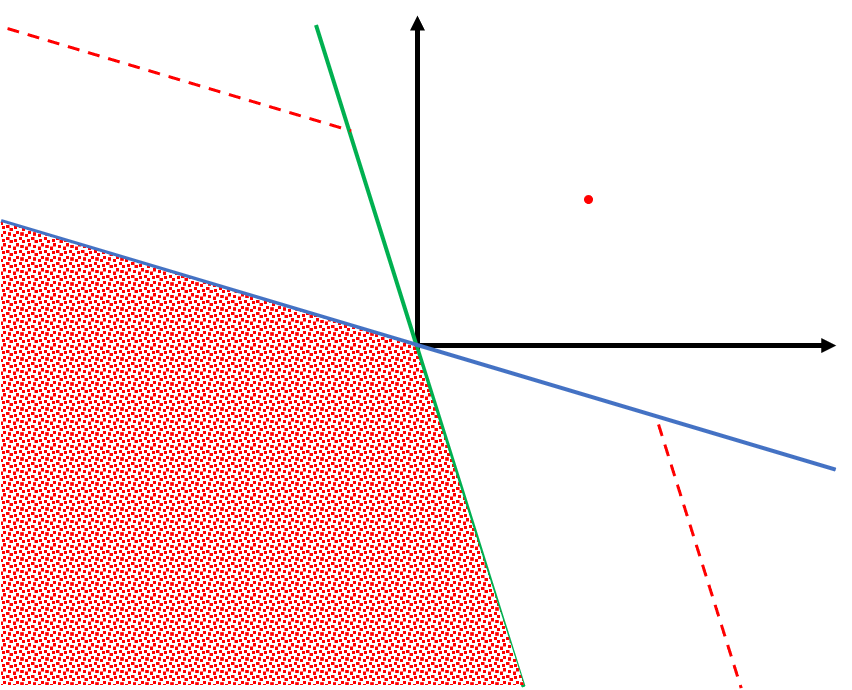}
        \caption{}
    \end{subfigure}
	\hfill
	\begin{subfigure}[b]{0.27\textwidth}
        \centering
        \includegraphics[width=\textwidth]{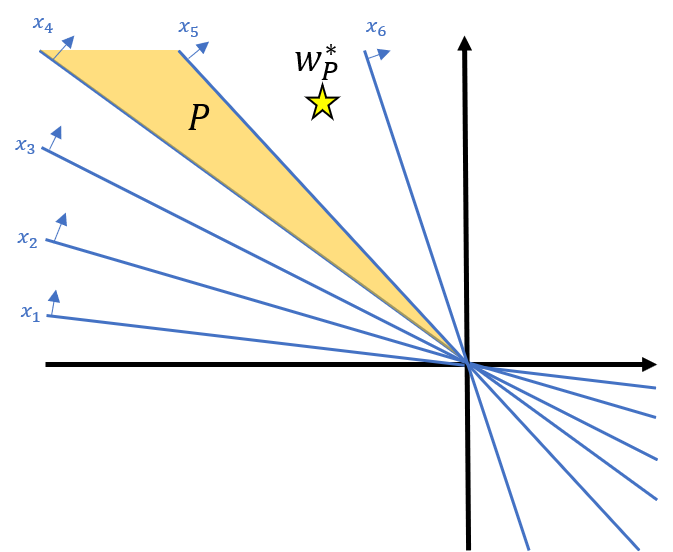}
        \caption{}
    \end{subfigure}
	\vspace*{-0.3cm}

	\caption{Description of terminology in single-neuron ReLU networks. 
	    In (a), a parameter $\wb$, an input data $\xb$, the activation boundary of $\xb$ and the solution hyperplane of $\xb$ in parameter space $\Wc$ are denoted by
	    the yellow point, a thick green vector, and the solid and dashed green lines, respectively. The activated half space is filled by light green color.
		In (b), stationary points of a single-neuron ReLU network with $n=d=2$ are denoted by red dots (a filled trapezoid, two dashed lines, and a point). 
		In (c), six data $\{\xb_i\}_{i=1}^6$ partition $\Wc=\Rd^2$. Note that 
		the virtual minimum of $P$ denoted by yellow star($\wb_P^*$) may not be contained in $P$.
    }
	\label{fig: example of terminology}
\end{figure*}

\paragraph{Training data set.}
For the analysis, we assume that the training data set $\{(\xb_i,y_i)\}_{i=1}^n$ composed of input vectors $\xb_i\in \Rd^d$ and output labels $y_i\in \Rd$ have the following properties:
\vspace{-5pt}
\begin{center}
\begin{minipage}{.4\textwidth}
    \begin{enumerate}[label=\textbf{A\arabic*}]\setlength{\itemsep}{-1mm}
        \item $ \xb_i \ge \zerob, \norm{\xb_i} \neq 0,~\forall i$. \label{asm: input} 
        \item $ y_i >0,~\forall i$. \label{asm: label}
        \item $\mathrm{rank}(\left[\xb_1 \cdots \xb_n\right]) = d$. \label{asm: underparameterization}
    \end{enumerate}
\end{minipage}
\end{center}
These assumptions can be justified as follows.
Since the input of each intermediate layer in deep ReLU networks is the output of the previous ReLU layer, we often use \ref{asm: input}. 
Similarly, Lemma~\ref{prop: basic properties} shows \ref{asm: label} is appropriate. Finally, the reduction principle for ReLU networks (Corollary \ref{cor: ReLU reduction principle}) leads to \ref{asm: underparameterization} without loss of generality.

\section{Partitions and support vectors of single-neuron ReLU networks}

The main goal of this section is to extend the analysis of the loss landscape of single-neuron linear networks in Appendix~\ref{sec: linear}  to 
single-neuron ReLU networks. 
This is thanks to the {\em label-backpropagation} described in Appendix \ref{sec: label backpropagation}, where each intermediate layer of deep ReLU networks can 
be considered as a single-neuron ReLU network.


Due to the existence of ReLU, one of important tools for ReLU network analysis is understanding parameter space partition.
Specifically, let $\Wc \subset \Rd^d$ denote the parameter space, i.e.,
$\wb\in \Wc$ for all network weights $\wb$.
Inspired by \citet{hanin2019complexity, hanin2019deep, lacotte2020all, xu2021traversing}, we refer a {\em partition} $P \subset \Wc$ as a subset of $\Wc$ that has invariant activation pattern.
For the parameter space $\Wc$ and an input data $\xb$, we define the {\em activated half space} with respect to $\xb$ as $\{ \wb\in \Wc~|~ \wb^T\xb > 0 \}$. 
Similarly, the subset $\{ \wb\in\Wc ~|~ \wb^T\bfx < 0 \}$ is called the {\em deactivated half space}, and $\{\wb\in \Wc~|~ \wb^T\bfx =0  \}$ is referred to the {\em activation boundary} with respect to $\xb$.
In addition, for a given data pair $(\bfx,y)$, the {\em solution hyperplane} of $\bfx$ is the hyperplane defined by $\{\wb \in \Wc ~|~\wb^T\bfx=y \}$
(Figure~\ref{fig: example of terminology}(b)).
Finally, for a given partition $P$, we say data $\xb$ is {\em activated} in $P$ and denoted by $\xb \sim P$ if $\wb^T\xb > 0$ for any $\wb \in P$. Similarly, a data $\xb$ is called {\em deactivated} in $P$ and denoted by $\xb \not\sim P$ if $\wb^T\xb \le 0$ for any $\wb \in P$.
%
%
{
The following {proposition} shows the necessary and sufficient conditions of activation on gradient flow.
\begin{proposition} \label{prop: activation}
    Consider a gradient flow $\wb(t)$ defined by \eqref{eq: ReLU gradient flow}. Then, gradient flow $\wb(t)$ deactivates $\xb$ at $t=s$ if and only if 
    \begin{align*}
        \xb^T\wb(s) = 0
        \qquad\text{and}\qquad 
        \nabla L(\wb(s))^T \xb \ge 0.
    \end{align*}
\end{proposition}

From Proposition \ref{prop: activation} with assumption \ref{asm: label}, we reveal one property about local minimum: every local minimum of \eqref{eq: cost ReLU} is strictly contained in some partition.}

\begin{lemma}[Not on boundary lemma]
\label{lem: not on boundary lemma}
	For a single-neuron ReLU network under \ref{asm: label}, {\eqref{eq: cost ReLU} has} no minimizer on any activation boundary.
\end{lemma}
Now, consider the associated loss function of a partition $P$ {defined by}
    \begin{align*}
        L_P(\wb) := \frac12\sum\limits_{\xb_i \sim P} (\wb^T\bfx_i - y_i)^2 .
    \end{align*}
{We refer $\wb_P^* := \argmin_{\wb\in \Wc} L_P(\wb) $ as the {\em virtual minimizer} of $P$. Notifying}
$\wb_P^*$ may not be contained in $P$ (see Figure \ref{fig: example of terminology}(c)), the following proposition states the precise condition when $P$ contains its virtual minima.

\begin{proposition}
\label{prop: ReLU solution}
	Suppose $\Hb(\wb)$ in \eqref{eq: Hq} has rank $r_P$ on a partition  $P$.
	If $r_P=d$, i.e. the rank is the same as the training sample dimension,  the unique virtual minimum of $P$ is given by
	\begin{align} \label{eq: virtual}
    	\wb_P^* =\left(\sum_{\xb_i \sim P} \xb_i\xb_i^T \right)^{-1} \sum_{\xb_i \sim P}y_i\xb_i 
	\end{align}
	and it is contained in $P$ if and only if $\wb_P^*$ satisfies
	\begin{equation} \label{eq: partition condition}
    	\begin{aligned}
    		\xb_i^T \wb_P^* > 0 &\qquad \textup{for all }\; \xb_i \sim P, \\
    		\xb_i^T \wb_P^* \le 0 &\qquad \textup{for all }\; \xb_i \not\sim P.
    	\end{aligned}
	\end{equation}
	If $r_P < d$, the virtual minima of $P$ exist in $P$ and they form a $(d-r_P)$-dimensional connected {affine subspace} in $P$.
\end{proposition}

\begin{figure}[t!]
    \centering
    \begin{subfigure}[b]{0.45\textwidth}
    	\centering
        \includegraphics[width=\textwidth]{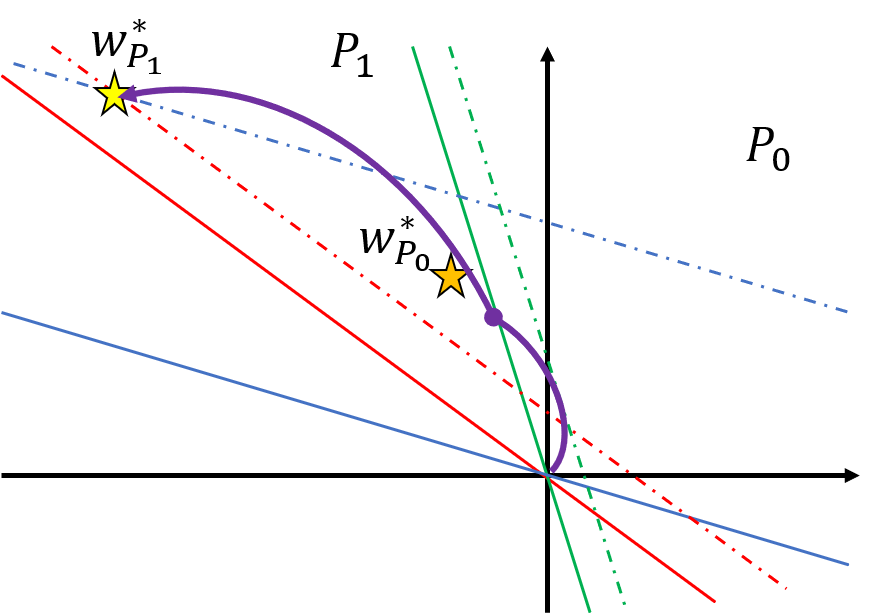}
    	\caption{} 
    	\label{fig: generalization}
    \end{subfigure}
    \hfill
    \begin{subfigure}[b]{0.4\textwidth}
        \centering
        \includegraphics[width=\textwidth]{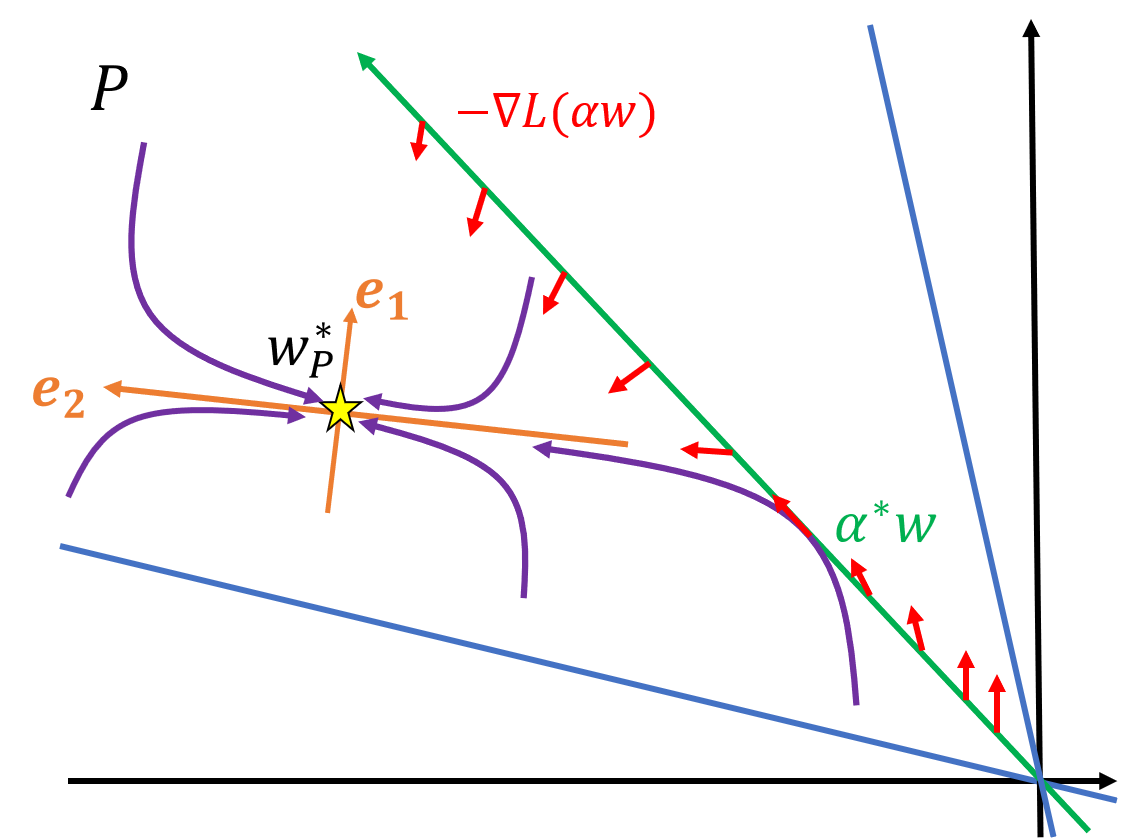}
        \caption{} 
        \label{fig: effect of norm}
    \end{subfigure}
    \caption{
    {{Training} dynamics of gradient flows}.
    In (a), three data induce activation boundaries(solid lines) and solution hyperplanes(dashed lines) in the parameter space $\Wc$. The virtual minimum of $P_0$ is denoted by the orange star($\wb_{P_0}^*$). After the gradient flow(purple curve) deactivates the green data, the virtual minimum is changed to the yellow star ($\wb_{P_{1}}^*$), which the gradient flow finally converges. 
    In (b), for a given vector $\alpha \wb$ in a partition $P$ with $\alpha>0$, small red arrows indicate the gradient $-\nabla L(\wb)$ at each point.
    By Proposition \ref{prop: effect of norm}, red arrows tend to activate other data for $\alpha<\alpha^*$. Similarly, for $\alpha > \alpha^*$, red arrows tend to deactivate other data. Some gradient flows are described by purple curves. 
    }
\end{figure}

{Proposition~\ref{prop: ReLU solution} leads to  an interesting concept of {\em support vectors}.} 
Specifically, for a partition $P$ such that the virtual minimum is contained in $P$, i.e. $\wb_P^* \in P$, we call $\xb \sim P$ as the {support vectors} of $\wb_P^*$.
Then, Proposition \ref{prop: ReLU solution} says that the data needed to compute $\wb_P^*$ is only its support vectors. 
It is worth noting that this terminology is closely related with the support vector machine \cite{cortes1995support, drucker1997support, vapnik1997support} and support vectors in linear regression \cite{kavitha2016comparative, joki2020clusterwise}, in the sense that support vectors are the only required data to obtain the solution.
Furthermore, the following result suggests an important advantage of ReLU networks  in terms of support vectors.
\begin{theorem} \label{thm: ReLU get smaller loss}
	Under \ref{asm: label}, the global minimum of a single-neuron ReLU network {has smaller loss} than a single neuron linear network. 
\end{theorem}
This states that in order to have a smaller loss, some data could be deactivated during the training of ReLU networks.
The following toy example shows the deactivation of data during the training of a single neuron ReLU networks.

\begin{example}[Deactivation in ReLU networks.] \label{exmp: deactivation interpretation}
    Consider the three data ($n=3$) in $\Rd^2$ ($d=2$) described in Figure \ref{fig: generalization}. 
    Let $P_0$ be the partition where all three data are activated. 
    Since the virtual minimum $\wb_{P_{0}}^*$ (orange star) is in the deactivated half space of the green data, the gradient flow must deactivate it (at the purple dot) and move to the next partition $P_1$. Then, the virtual minimum is changed to the yellow star ($\wb_{P_1}^*$) which is the optimal solution of two data (blue and red). i.e., the gradient flow finally disregards the green data and converges to the optimal point of the remained data. See also Example \ref{exmp: deactivation} in Section \ref{sec: experiments}.
\end{example}

Theorem~\ref{thm: ReLU get smaller loss} and Example~\ref{exmp: deactivation interpretation}
clearly show that a gradient flow gives up to fit some data and rather focuses on the best-fit of the remained data (i.e., support vectors) to learn larger common tendency of data. 
This may explain {why {and how} ReLU networks generalize better than linear networks}.
Then, one may wonder whether the large number of deactivation (i.e. smaller number of support vectors) is preferable. However, the following theorem says that 
if several local minima with different number of support vectors are feasible, the one with more support vectors is preferrable.

\begin{theorem}
\label{thm: more support vectors}
	Consider a single-neuron ReLU network {\eqref{eq: cost ReLU}} with \ref{asm: label}. Let $\wb_1^*$ be a local minimum which has set of support vectors $S_1$, and $\wb_2^*$ be another local minimum which has set of support vectors $S_2$ such that $S_2 \subset S_1$. Then, $L(\wb_2^*) \ge L(\wb_1^*)$.
\end{theorem}
 Later in Theorem \ref{thm: gradient flow does not converge to a bad minimum}, we will show  that the gradient flow dynamics tends to avoid bad local minima and converges to the one with a large number of support vectors, which is another important implicit bias of the gradient {flow}.

\section{Gradient flow dynamics}

\subsection{Weight initialization}

Recall Figure \ref{fig: generalization}  where a gradient flow initialized with small norm goes through the partition with all the data being active, after which some of the data become deactivated to reach a local minimizer. It turns out that weight initialization plays the key roles in this dynamics.
In particular, the norm of the initial weight plays  key role to the activation of data, as explained in the following proposition.


\begin{proposition}
    \label{prop: effect of norm}
	For a single-neuron ReLU network \eqref{eq: cost ReLU}, consider $\wb_0$ in a partition $P$ and define
	\vspace{-6pt}
	\begin{align*}
	    \alpha_j^* := \frac{\xb_j^T\qb}{\xb_j^T \Hb \wb_0}
	\end{align*}
	for a data $\xb_j$, where  $\Hb$ and $\qb$ are given by \eqref{eq: Hq}.
	Then, $\nabla L(\alpha\wb_0)^T\xb_j \ge 0 $ if and only if $\alpha \ge \alpha_j^*$.
\end{proposition}

{
Proposition \ref{prop: effect of norm} leads to a conclusion that a gradient flow with sufficiently large norm (i.e. $\alpha \ge \alpha_j^*$) satisfies the deactivation condition of Proposition \ref{prop: activation}. 
On the other side, for a gradient flow with small norm, we can expect that all data being activated, which is indeed true as shown in Lemma \ref{lem: all activated}. See also Figure \ref{fig: effect of norm}.
}

{
Activation of data is a significant issue since Theorem \ref{thm: more support vectors} guarantees lower loss value for larger number of support vectors. 
}
%
%
%
In the following theorem, we further suggest a condition of initialization point $\wb_0$ such that one specific data $\xb_j$ is kept activated on the gradient flow $\wb(t)$.

\begin{theorem}[No Deactivation] \label{thm: ReLU no deactivation}
	Consider a gradient flow \eqref{eq: ReLU gradient flow} in a single-neuron ReLU network under \ref{asm: input} and \ref{asm: label}. Suppose there exists $\wb_{GM}^*$ such that  $L(\wb_{GM}^*)=0$.
    Then,  $\xb_j$ is {always} activated on the gradient flow initialized {at $\wb_0$} if 
    $\xb_j^T\wb_0>0$ and
	\begin{align}\label{eq: no deactivation condition}
	   \frac{y_j}{ \norm{\xb_j}}> \norm{\wb_0 - \wb_{GM}^*} .
	\end{align}
	In particular, if \eqref{eq: no deactivation condition} holds for all $j=1,\cdots,n$, then the gradient flow initialized at $\wb_0$ coincides with that of a single-neuron linear network initialized at the same point $\wb_0$, which converges to the global minimum. 
\end{theorem}

{The existence of $\wb_{GM}^*$ that activates all data in Theorem~\ref{thm: ReLU no deactivation} is easily satisfied in the overparameterized neural network. 
If \eqref{eq: no deactivation condition} does not hold for some $1 \le j\le n$}, gradient flows of linear and ReLU networks do not coincide and the gradient flow of the ReLU network may not converge to the global minimum (see Example \ref{exmp: deactivation}).
Nonetheless, the following theorem shows that a gradient flow initialized under some conditions does not converge to a bad local minimum.

\begin{theorem}[Gradient flow does not converge to a bad local minimum]
\label{thm: gradient flow does not converge to a bad minimum} 
	Consider a gradient flow \eqref{eq: ReLU gradient flow} under the same condition of Theorem \ref{thm: ReLU no deactivation}. Let $\wb_{loc}^*$ be a local minimum and $S$ be the set of its support vectors. 
	Now suppose {the initialization point $\wb_0$} satisfies
	\begin{align}\label{eq: condition for not bad minimum}
		\max_{\xb_j \in S^c} \left[ \frac{y_j}{\norm{\xb_j}} \right] 
		\ge \norm{\wb_0-\wb_{GM}^*}.
	\end{align}
	Then, the gradient flow initialized at $\wb_0$ does not converge to $\wb_{loc}^*$.
\end{theorem}

For example, consider two local minima $\wb_1^*$ and $\wb_2^*$ with their sets of support vectors $S_1$ and $S_2$ such that $S_2 \subset S_1$. Then,
$$ \max_{\xb_j \in S_1^c} \left[ \frac{y_j}{\norm{\xb_j}} \right] 
\le \max_{\xb_j \in S_2^c} \left[ \frac{y_j}{\norm{\xb_j}} \right] 
$$
implies that \eqref{eq: condition for not bad minimum} looks more feasible for $S_2$ than $S_1$. This suggests that a gradient flow may not converge to a `bad' local minimum that has few support vectors without crucial data in the sense of {Theorem \ref{thm: more support vectors}. That is, we can say that a gradient flow does not converge to a local minimum with large loss value.}
There is another interpretation of this theorem. Since $L(\wb_{GM}^*)=0$, i.e., $y_j=\xb_j^T\wb_{GM}^*$ for all $j$, \eqref{eq: condition for not bad minimum} can be converted to
\begin{align} \label{eq: cos condition}
    \max\limits_{\xb_j \in S^c} \cos \theta_j 
    > \frac{\norm{\wb_0-\wb_{GM}^*}}{\norm{\wb_{GM}^*}}
\end{align}
where $\theta_j$ is the angle between $\xb_j$ and $\wb_{GM}^*$, i.e. $\cos \theta_j = \frac{\xb_j^T \wb^*_{GM}}{\norm{\xb_j}\cdot \norm{\wb_{GM}}}$. Then \eqref{eq: cos condition} says that a data $\xb_j$ which has large $\cos \theta_j$ value is not deactivated on the gradient flow, which is a candidate of crucial data.

In terms of Proposition \ref{prop: effect of norm}, the gradient flow with small norm tends to activate all data and the one with large norm does the opposite.
Therefore, the data which aligns well on the $\wb_{GM}^*$ easily satisfies \eqref{eq: cos condition} and would be kept activated after it is activated with the help of small norm initialization.
On the other hand, the gradient flow with large norm initialization does not allow the data to satisfy \eqref{eq: cos condition} and may converge to the minimum with few support vectors.
In the middle of two realms, there is a critical region where {$\norm{\wb_0} \approx \norm{\wb_{GM}^*}$}. If the direction of $\wb_0$ is {close with} $\wb_{GM}^*$, then all data is activated, and the gradient flow easily converge to the global minimum. {In contrast, if the direction of $\wb_0$ is far from $\wb_{GM}^*$}, the right hand side of \eqref{eq: cos condition} is much larger. Thus, we can conjecture that the convergence is sensitive {to the direction of $\wb_0$ in this critical region.}
Accordingly, our observations extends the existing results on why  the gradient flow initialized with infinitesimally small norm is preferred in the gradient flow dynamics \cite{gunasekar2017implicit, razin2020implicit, arora2019implicit, woodworth2020kernel, li2020towards}. 

 Finally, it is worth noting that even though $\zerob \in \Wc$ is a cusp, every gradient flow initialized with infinitesimally small norm converges to the same local minimum.
\begin{proposition} \label{prop: zero initialization unique}
	Consider a gradient flow \eqref{eq: ReLU gradient flow} of a single-neuron ReLU network under \ref{asm: input}, \ref{asm: label}, and \ref{asm: underparameterization}. Further suppose that the Hessian matrix $\Hb(\wb(t))$ has full rank on the gradient flow until it converges. Then, there exists $\delta>0$ such that for all $\wb_0$ with  $\norm{\wb_0}<\delta$ and $-\nabla L(\wb_0) > \zerob$, every gradient flow initialized at $\wb_0$ converges to the same point.
\end{proposition}

\subsection{Norm increasing property}
In Appendix \ref{sec: linear}, we show
 that linear regression has implicit biases that the gradient flow initialized at zero converges to the minimum norm solution, with strictly increasing its norm until it converges. In Theorem~\ref{thm: norm increasing}, we extend this norm increasing property to single-neuron ReLU networks.
By the {balancedness property} of deep ReLU networks shown by \citet{du2018algorithmic}, the norm of each intermediate layers increase together. Therefore, if one intermediate layer and its backpropagated labels defined in Appendix \ref{sec: label backpropagation} satisfies the conditions of Theorem \ref{thm: norm increasing}, then every layer has norm increasing property together.
Also note that this norm increasing property does not contradict with the result of \citet{vardi2021implicit}, since it does not need to converge to the minimum norm solution.
We provide such example in Example \ref{exmp: deactivation}.

Although the conditions for Theorem \ref{thm: norm increasing} look complicated,
for the special case of $d=2$, the condition becomes trivial
as there exists a special structure in the partitioned parameter space. More precisely, there is an `order' for activation and deactivation, as the following lemma states.
\begin{lemma}[Ordering of partitions in $\Rd^2$]
\label{lem:ordering}
	Consider the partitioned parameter space $\Wc=\Rd^2$ under \ref{asm: input}. Then we can impose an relative order between partitions. In particular, for any two partitions $P_1$ and $P_2$ in the 2nd (or 4th) quadrant, either $\{\xb_i ~|~ \xb_i \sim P_1 \} \subset \{ \xb_i ~|~ \xb_i \sim P_2 \}$ or $\{\xb_i ~|~ \xb_i \sim P_1 \} \supset \{ \xb_i ~|~ \xb_i \sim P_2 \}$ holds.
\end{lemma}

This provides some useful geometric insights to understand learning dynamics of gradient flow. 
We start with introducing an interesting lemma. 

\begin{lemma}[No revisit lemma for $d=2$] \label{lem: no revisit}
	Consider a single-neuron ReLU network with $d=2$ under \ref{asm: input}, \ref{asm: label}, and \ref{asm: underparameterization}. Then a gradient flow initialized with infinitesimally small norm does not re-activate any deactivated data.
\end{lemma}

With this lemma, we can obtain the global convergence and norm-increasing property in $\Rd^2$. 

\begin{theorem}[Global convergence for $d=2$] \label{thm: 2D input}
	Consider a single-neuron ReLU network with $d=2$ under \ref{asm: input}, \ref{asm: label}, and \ref{asm: underparameterization}. Then, a nontrivial ($\qb(\wb_0) >\zerob $) gradient flow initialized with infinitesimally small norm converges to the global minimum with strictly increasing its norm.
\end{theorem}


\section{Experiments} \label{sec: experiments}
In this section, we provide some empirical examples for the results we proposed theoretically. 
 Detail of these experiments is described in Appendix \ref{app: experiments}. 

{
We first observe the effect of the initialization norm.} Recall that
Theorem \ref{thm: norm increasing} and \ref{thm: 2D input} consider gradient flows initialized with infinitesimally small norm. {It is shown in the following example that the condition of small norm initialization is necessary. }

\begin{figure}[ht!]
	\centering
	\begin{subfigure}[b]{0.45\textwidth}
        \centering
        \includegraphics[width=\textwidth]{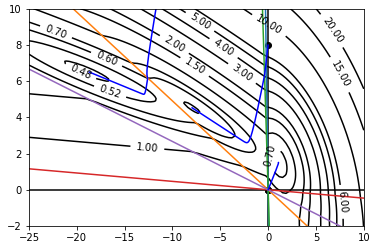}
        \caption{}
    \end{subfigure}
    \hfill
	\begin{subfigure}[b]{0.45\textwidth}
        \centering
        \includegraphics[width=\textwidth]{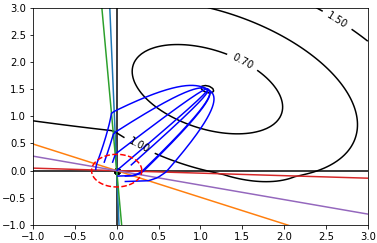}
        \caption{}
    \end{subfigure}
	\caption{
	Loss landscape of single-neuron ReLU networks.
	In (a), loss level curves in $\Wc=\Rd^2$ (black curves) and gradient flows with three distinct initialization points (blue curves). 
	Boundaries of partitions are drawn by colored solid lines. Note that the gradient flow initialized with infinitesimally small norm converges to the global minimum, while other gradient flows converge to local minima. See Proposition \ref{prop: effect of norm} and Theorem \ref{thm: 2D input}.
	{In (b), gradient flows initialized with $\norm{\wb_0}<0.3$ are displayed. As shown in Proposition \ref{prop: zero initialization unique} and Theorem \ref{thm: 2D input}, they converge to the same point, which is the global minimum.}
	}
	\label{fig: initialized at zero is important}
\end{figure}

\begin{example}[Initialization with infinitesimally small norm is necessary] \label{exmp: initialization with infinitesimally small norm}
	Consider a single-neuron ReLU network with $d=2$ and $n=5$. Detail data setting is described in Appendix \ref{subsec: exmp deactivation}. In Figure~\ref{fig: initialized at zero is important}, the level curves of the loss function \eqref{eq: cost ReLU} and gradient flows initialized with three different points are plotted by the black and the blue curves, respectively. The initialization points are 
	denoted by black points (one point is out of scope). 
	{
	With regard to Proposition \ref{prop: effect of norm}, Figure~\ref{fig: initialized at zero is important} illustrates that gradient flows of small norm tends to activate data and converge to local minima with many support vectors, while ones of large norm does the opposite.
	In addition, considering Proposition \ref{prop: zero initialization unique} and Theorem \ref{thm: 2D input}, we can observe the gradient flows initialized with infinitesimally small norms converge
	to the same point {which is the} global minimum (by Theorem \ref{thm: more support vectors}), while other gradient flows deactivate some data and converge to local minima.
	This {shows how} the convergence of gradient flows depends on the norms of initial points.
	}
\end{example}

The next example exhibits the case where the assumption of Theorem \ref{thm: ReLU no deactivation} does not hold, thus a data can be deactivated.
Moreover, this is a counter example of Theorem \ref{thm: 2D input} for $d>2$, {thus we both prove and disprove {the global convergence} for all $d \ge 2$}. 

\begin{figure}[ht!]
	\centering
	\begin{subfigure}[b]{0.4\textwidth}
        \centering
        \includegraphics[width=\textwidth]{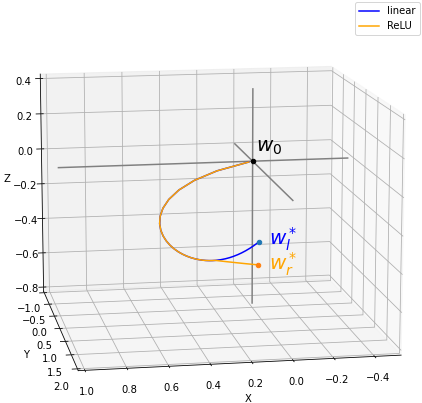}
        \caption{}
    \end{subfigure}
	\hfill
	\begin{subfigure}[b]{0.45\textwidth}
        \centering
        \includegraphics[width=\textwidth]{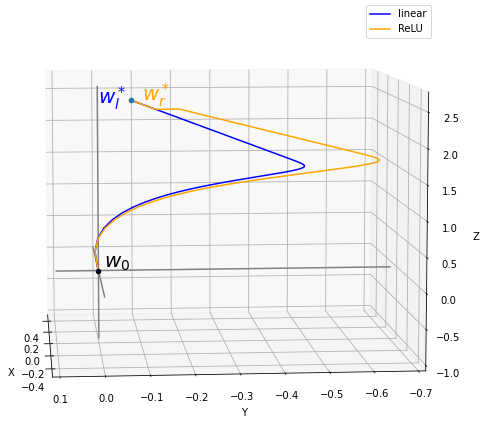}
        \caption{}
    \end{subfigure}
	\caption{Gradient flows in Example \ref{exmp: deactivation} and Example \ref{exmp: reactivation}.
	In (a), gradient flows of single-neuron linear and ReLU networks in Example \ref{exmp: deactivation} are plotted. Note that they coincide at the beginning, but bifurcate after one data deactivated. 
	In (b), gradient flows of single-neuron linear and ReLU networks in Example \ref{exmp: reactivation} are plotted. Although there is deactivation and gradient flows are distinct, they converge to the same point, which shows a reactivation.
	}
	\label{fig: examples of results}
\end{figure}

\begin{example}[Deactivation of single-neuron linear and ReLU network] \label{exmp: deactivation} 
    Consider a single-neuron ReLU network for $n=d=3$, with data $\{(\xb_i, y)\}_{i=1}^3$ is given in Appendix \ref{subsec: exmp deactivation}.
    For each $\xb_i$, define $h_i(\wb):=\wb^T\xb_i$. Then $\{\wb\in\Wc ~|~ h_i(\wb)=0\}$ is the activation boundary of $\xb_i$, and $\{\wb\in\Wc ~|~ h_i(\wb)=y_i\}$ is the solution hyperplane of $\xb_i$.	
    There is a unique global minimum for both linear and ReLU networks, which is in the all-activated partition. 
	However, we can observe that the gradient flow of the ReLU network deactivates $\xb_1$ during training. Note that gradient flows of the linear and ReLU networks coincide first, but bifurcate after deactivation of $\xb_1$ (See Figure \ref{fig: examples of results}(a)). Since $\xb_1$ is not reactivated again until the gradient flow converges, it is not a support vector of convergent local minimum.
	See Figure \ref{fig: examples of results}(a) and Figure~\ref{fig: deactivation} for more analysis of the examples.
\end{example}

In the next example, we show re-activation may occur for $d>2$, which shows that Lemma \ref{lem: no revisit} is {the best result can be obtained under the assumptions.}

\begin{example}[Reactivation of single-neuron linear and ReLU networks, for $d>2$] \label{exmp: reactivation} 
	Consider a single-neuron ReLU network for $n=4$ and $d=3$ with data $\{(\xb_i, y)\}_{i=1}^4$ which is given in Appendix \ref{subsec: exmp reactivation}.
	 For each $\xb_i$, define $h_i(\wb):=\wb^T\xb_i$ like in Example \ref{exmp: deactivation}. Then we notice that the gradient flow of the ReLU network deactivates $\xb_4$ soon, and reactivates it later (see Figure \ref{fig: reactivation}(f)). Since all data are activated at the last, gradient flows of ReLU and linear networks converge to the same point, which is the unique global minimum. Note that the trace of two gradient flows are quite different although they {coincide at the initial point and convergent points.} See Figure \ref{fig: examples of results} (b). 
    See Figure~\ref{fig: reactivation} for the detail result.
\end{example}

\section{Conclusion and future work}
Understanding implicit bias of gradient descent has been an important goal in machine learning research. In this paper, 
we investigated implicit bias of gradient flow dynamics in single-neuron ReLU networks with square loss and provided following observations.
First, we showed the implicit bias of gradient flows in terms of support vectors of ReLU networks to answer why and how ReLU networks generalize well.
Second, we revealed an implicit bias of gradient flow dynamics with respect to the norm of initialization. Specifically, we provided an initialization condition when a gradient flow keeps some data activated. Using this, we showed that a gradient flow with some condition on initialization does not converge to bad local minima.
Third, we extended the norm-increasing property of single-neuron linear networks to single-neuron ReLU networks under some conditions, revealing another implicit bias of gradient flows.
{Finally, for a special case $d=2$, we proved the global convergence of the gradient flow.}

\paragraph{Limitations.} 
This work is not free of limitations. 
It is worth to seek conditions so that Theorem \ref{thm: 2D input} can be generalized to guarantee the global convergence in higher dimension $d>2$.
Second, perhaps using the label-backpropagation we proposed in Appendix~\ref{sec: label backpropagation}, we should generalize this result in deep ReLU networks. 

\newpage
\bibliographystyle{icml2022.bst}
{\small
\bibliography{reference.bib}
}

\newpage
\appendix
\section*{\huge Appendix}

\section{Preliminaries} \label{sec: preliminaries}

For a square matrix $\Ab$, matrix exponential $e^\Ab$ is defined by
$$ e^\Ab := \sum_{m=0}^{\infty} \frac{1}{m!}\Ab^m.
$$
For a positive semidefinite matrix $\Ab$, if $\Ab$ is not invertible, we define its pseudo-inverse
$$ \Ab^\dagger := \sum_k \frac{1}{\lambda_k} \eb_k\eb_k^T,
$$ 
where $(\lambda_k, \eb_k)$ are positive eigenvalues and corresponded eigenvectors of $\Ab$. We use $\norm{\cdot}_F$ to denote Frobenius norm of a matrix. $\odot$ denotes elementwise multiplication between two vectors. For a finite set $S$, $|S|$ means the number of elements of $S$. We denote $B_r (\xb) := \{\yb: \norm{ \yb-\xb} < r\}$. 

\section{Multi-neuron ReLU networks}\label{sec: multineuron}
{
In this section, we prove that the result of single neuron ReLU networks can be directly extended to sigle-layer multineuron case. Let $f_\Wb(\xb) = [\Wb\xb]_+$ be a single layer multineuron case with input dimension $d_{in}$ and output dimension $d_{out}$. Let $\wb_1^T, \wb_2^T, \cdots, \wb_{d_{out}}^T$ be the row vectors of $\Wb$.
Then, the loss function \eqref{eq: cost ReLU} becomes
\begin{align*}
    L(\Wb) &= \frac12 \sum_{i=1}^n \norm{[\Wb\xb_i]_+ -\yb_i}^2 \\
    &= \frac12 \sum_{i=1}^n \sum_{j=1}^{d_{out}} ([\wb_j^T\xb_i]_+ -(\yb_i)_j)^2 \\
    &= \frac12 \sum_{j=1}^{d_{out}} \sum_{i=1}^n ([\wb_j^T\xb_i]_+ -(\yb_i)_j)^2.
\end{align*}
Thus for each row vector $\wb_j$, its gradient is given by
\begin{align*}
    \frac{\partial L}{\partial \wb_j } = \sum_{i=1}^n ([\wb_j^T\xb_i]_+-(\yb_i)_j) \xb_i,
\end{align*}
which is equivalent to the loss of single-neuron ReLU networks.
Therefore, a single-layer multi-neuron ReLU network is just a set of independent $d_{out}$ single-neuron ReLU networks, and the result of single-neuron ReLU networks (Lemma \ref{lem: not on boundary lemma}, Theorem \ref{thm: ReLU get smaller loss}, Theorem \ref{thm: gradient flow does not converge to a bad minimum}, Theorem \ref{thm: norm increasing}, Lemma \ref{lem: no revisit}, Theorem \ref{thm: 2D input}) can be directly generalized to single-layer multi-neuron ReLU networks.
}

\section{Label-backpropagation of deep ReLU networks} \label{sec: label backpropagation} 
In this section, we show how single-neuron linear and ReLU networks form building blocks of deep ReLU networks.
\begin{figure*}[ht!]
	\centering
	\includegraphics[width=0.8\linewidth]{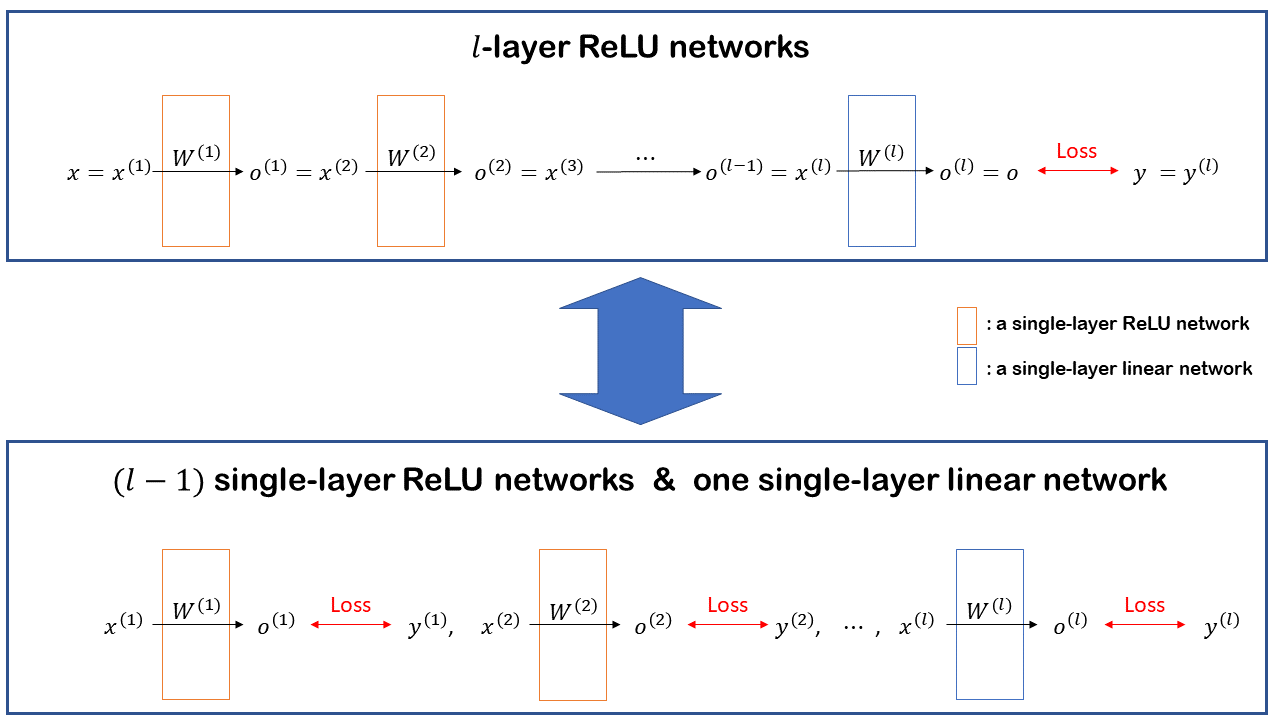}
	\caption{Label-backpropagation. For a given data pair $(\xb,\yb)$, the gradient of an $l$-layer ReLU network is equivalent to the gradient of $(l-1)$ single-neuron ReLU networks and one single-neuron linear network, where data pair is given by Proposition \ref{prop: label backpropagation}.}
	\label{fig: Label backpropagation}
\end{figure*}
\begin{proposition}[Label-backpropagation] \label{prop: label backpropagation}
	Consider an $l$-layer ReLU network $f(\cdot \; ; \W{1}, \cdots, \W{l})$ with one data pair $(\bfx, \yb)$ under square loss. Let $\x{i}$ and $\o{i}$ be input and output of $i$-th layer $ (i=1,2,\cdots,l-1)$, which are defined by
	\begin{align*}
	    \o{i} &:= [\W{i}\x{i}]_+ \\
	    \x{i} &:= \o{i-1}
	\end{align*}
	where $\o{0} := \xb$, $\x{l} := \o{l-1}$ and $\o{l}:=\W{l}\x{l}$.
	Define $\d{l} := \o{l}-\y{l}$ and $\y{l} := \yb$. Now for $m=1,2,\cdots,l-1$, recursively define $\d{m}$ and backpropagated-label $\y{m}$ by
	\begin{align*}
		\d{m} &:= \mm_{\{ \x{m+1}>\zerob \}} \odot \WT{m+1}\d{m+1}, \\
		\y{m} &:= \o{m}-\d{m}.
	\end{align*}
	Then the gradient of $\W{i}$ of the $l$-layer ReLU network given by the data pair $(\xb, \yb)$ is the same with the gradient of $\W{i}$ of a single-neuron ReLU network $\f{i}$ given by the data pair $(\x{i}, \y{i})$ for $i=1,2,\cdots,l-1$. 
\end{proposition}
\begin{proof} \allowdisplaybreaks
	An $l$-layer ReLU network $f(\cdot \; ; \W{1}, \cdots, \W{l})$ is modeled by
	\begin{align*} 
	    f(\xb \;; \W{1}, \cdots\!, \W{l}) = 
	    \W{l} [ \cdots \W{2}[\W{1}\xb]_+ \cdots ]_+
	\end{align*}
	
	For a given data pair $(\xb, \yb)$, the square loss function is defined by
	$$ L( \W{1}, \cdots\!, \W{l} ) \!=\! \frac12 \norm{f(\xb \;; \W{1}, \cdots\!, \W{l})-\yb}_F^2.
	$$
	Then, the gradient of $\W{m}$ is computed by
	\begin{align}
	    \frac{\p L}{\p \W{l}} &= \d{l} \xT{l}, \label{eq: last layer} \\
	    \frac{\p L}{\p \W{m}} &= \d{m} \xT{m} \quad \text{for}\quad  m=1,2,\cdots, l-1. \label{eq: intermediate layers}
	\end{align}
	See \citet{calin2020deep} for a detail derivation. Now we focus on the last layer $\W{l}$. Consider a single-neuron linear network $\f{l}(\cdot \; ; \W{l})$ with a given data pair $(\x{l}, \y{l})$. The square loss provides gradient of $\W{l} $ by
	\begin{align*}
	    \frac{\p}{\p \W{l}} \frac12 \norm{\f{l}(\x{l} \; ; \W{l})-\y{l}}_F^2
	    &= \frac{\p}{\p \W{l}} \frac12 \norm{\W{l}\x{l} -\y{l}}_F^2 \\
	    &= (\W{l}\x{l} - \y{l})\xT{l}\\
	    &= \d{l}\xT{l},
	\end{align*}
	which is equal to \eqref{eq: last layer}. Therefore, the single-neuron linear network $\f{l}(\cdot \; ; \W{l})$ with a given data pair $(\x{l}, \y{l})$ provides the same gradient \eqref{eq: last layer} for $\W{l}$.
	
	Similarly, we can apply this argument for intermediate layers. For $m=1,\cdots,l-1$, consider a single-neuron ReLU network $\f{m}(\cdot\;;\W{m})$ with a data pair $(\x{m}, \y{m})$. Then, the gradient of $\W{m}$ is computed by
	\begin{align*}
	    \frac{\p}{\p \W{m}} \frac12 \norm{\f{m}(\x{m} \; ; \W{m})-\y{m}}_F^2
	    &= \frac{\p}{\p \W{m}} \frac12 \norm{[\W{m}\x{m}]_+ -\y{m}}_F^2 \\
	    &=  \left(\mm_{\{\W{m}\x{m} > \zerob \}}\odot (\W{m}\x{m} - \y{m})\right) \xT{m} \\
	    &= \left(\mm_{\{[\W{m}\x{m}]_+ > \zerob \}}\odot ([\W{m}\x{m}]_+ - \y{m})\right)\xT{m} \\
	    &= \left( \mm_{\{\o{m} > \zerob \}}\odot (\o{m} - \y{m})\right) \xT{m} \\
	    &= \left( \mm_{\{\x{m+1} > \zerob \}}\odot (\o{m} - \y{m})\right) \xT{m} \\
	    &= \d{m}\xT{m},
	\end{align*}
	which is \eqref{eq: intermediate layers}. Therefore, the single-neuron ReLU network $\f{m}(\cdot\;;\W{m})$  with a data pair $(\x{m}, \y{m})$ provides the same gradient \eqref{eq: intermediate layers} for $\W{m}$.
	
	To sum up, gradient of the $l$-layer ReLU network $f$ can be equivalently obtained from $(l-1)$ single-neuron ReLU networks $\f{m}$ with $m=1,2,\cdots,l-1$ and one single-neuron linear network $\f{l}$.
\end{proof}
{
This proposition means that training $l$-layer ReLU networks can be understood as training $(l-1)$ single-neuron ReLU networks and one linear network. See Figure \ref{fig: Label backpropagation}.

\begin{lemma}[\citet{du2018algorithmic}] \label{lem: balancedness}
    For an $l$-layer ReLU network, on a gradient flow, the difference of Frobenius norm of weight matrices of adjoined layers is invariant, i.e.,
    $$ \frac{d}{dt} \left( \norm{\W{m}}_F^2 - \norm{\W{m+1}}_F^2 \right) = 0. $$
\end{lemma}

By this property of deep ReLU networks, if an intermediate single-layer ReLU network has norm-increasing property (Theorem \ref{thm: norm increasing}), then others also have it.
This is one reason why studying a single-layer ReLU network is crucial.
}

\section{Single-neuron linear networks} \label{sec: linear}
Single-neuron linear networks are trained as follows:
\begin{align} \label{eq: cost linear}
	\argmin_{\wb \in \Wc} L(\wb) := \frac12 \sum_{i=1}^n (\wb^T\xb_i-y_i)^2,
\end{align}
Then, the gradient flow initialized at $\wb_0$ is given by solution of the following differential equation 
\begin{align} \label{eq: linear gradient flow} 
	\frac{d\wb(t)}{dt} = - \nabla_\wb L = -\Hb\wb + \qb,
	\qquad \wb(0)=\wb_0.
\end{align}
where $\Xb=\begin{bmatrix} \xb_1 & \cdots & \xb_n \end{bmatrix}$ and
$$\Hb:=\sum\limits_{i=1}^n \xb_i \xb_i^T=\Xb\Xb^T,\quad \qb:=\sum\limits_{i=1}^n y_i\xb_i=\Xb\yb.$$

Since $\Hb = \sum\limits_{i=1}^n \xb_i \xb_i^T$ is positive semidefinite, $L$ is convex with respect to $\wb$. Therefore, the set of minima of \eqref{eq: cost linear} is equal to the set of stationary points of $L$.

\begin{proposition}[The manifold of stationary points]\label{prop: stationary point linear}
	Consider the loss function $L(\wb)$ in \eqref{eq: cost linear} and suppose the data matrix $\Xb$ has rank $r$. Then,
	\begin{enumerate}
	    \item [(i)] every gradient flow converges to the global minimum.
		\item [(ii)] If $r \ge d$, the stationary point is unique, which is the global minimum given by $\wb^*=\Hb^{-1} \qb$.
		\item [(iii)] Otherwise (i.e., $r<d$), the set of stationary points forms a $(d-r)$-dimensional connected linear manifold containing $\wb^*=\Hb^\dagger \qb$.
	\end{enumerate}
\end{proposition}
\begin{proof}
	This is a basic property of linear regression. See \citet{anton2003contemporary} for detail.
\end{proof}

\subsection{Norm-increasing property of single-neuron linear networks.}
The goal of this subsection is to prove the norm-increasing property of single-neuron linear networks.
Thanks to the reduction principle (Theorem \ref{thm: reduction principle}) which will be explained later, without loss of generality, we can assume $\Hb$ has the full rank. 

\begin{definition}[Hyperrectangle of $\Hb$] \label{def: hyperrectangle}
	Consider \eqref{eq: cost linear} and suppose the Hessian matrix $\Hb$ has full rank. For the unique global minimum $\wb^*= \Hb ^{-1}\qb$, let $\eb_1, \eb_2, \cdots \eb_d$ be the eigenvectors of $\Hb$ with direction such that $c_k^* := \eb_k^T \wb^* \ge 0$. Then, {\em the hyperrectangle of $\Hb$} is defined by
	$$ \left\{ \wb=\sum_{k=1}^d c_k\eb_k \in \Wc ~|~ 0 \le c_k \le c_k^* \right\} \subset \Wc.
	$$
\end{definition}

\begin{definition} \label{def: g(w)}
	Consider \eqref{eq: cost linear} and a gradient flow $\wb(t)$. Define a function $g(\wb):=-\wb^T \dot{\wb}$. Then, the {\em norm-increasing subset} is defined by the set
	$$ \{ \wb \in\Wc ~|~ g(\wb)<0 \} \subset \Wc. $$
\end{definition}

\begin{figure}[ht!]
	\centering
	\begin{subfigure}[b]{0.45\textwidth}
        \centering
        \includegraphics[width=\textwidth]{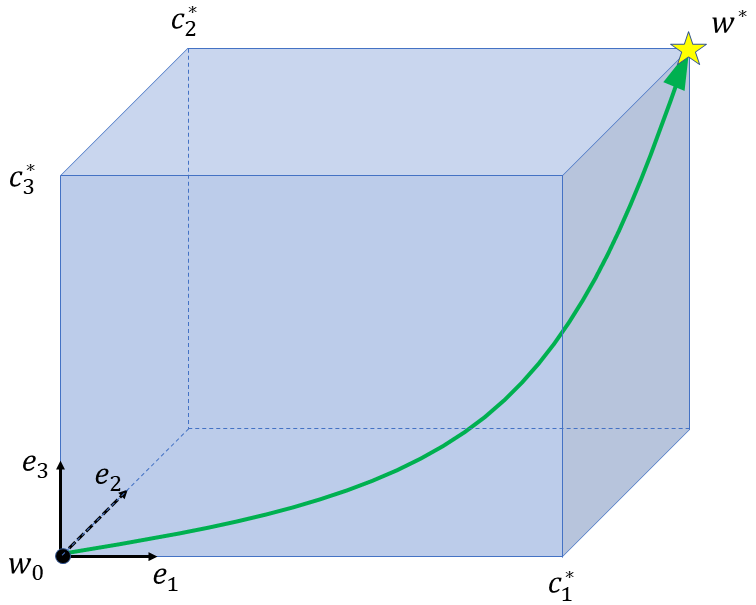}
        \caption{}
    \end{subfigure}
	\hfill
	\begin{subfigure}[b]{0.45\textwidth}
        \centering
        \includegraphics[width=\textwidth]{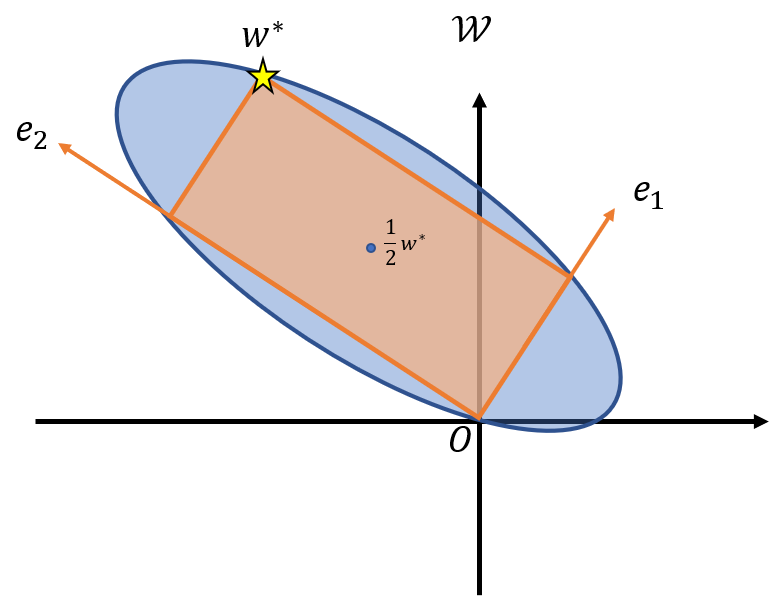}
        \caption{}
    \end{subfigure}
	\caption{Examples of hyperrectangle of $\Hb$ and norm-increasing subset. 
	(a) In $\Wc=\Rd^3$, $\eb_1, \eb_2$ and $\eb_3$ are eigenvectors of $\Hb$. The global minimum $\wb^*$ is denoted by a yellow star, and the green curve is a gradient flow initialized at $\zerob$. Note that the gradient flow converges to $\wb^*$ without escaping the hyperrectangle.
	(b) In $\Wc=\Rd^2$, a norm-increasing subset is described by the blue solid ellipse, where the hyperrectangle of $\Hb$ is described by the orange rectangle. Note that the axes of the rectangle and the norm-increasing ellipse are parallel to eigenvectors of $\Hb$, and have the same center $\frac12 \wb^*$.}
	\label{fig: hyperrectangle}
\end{figure}

The following proposition explains why the hyperrectangle of $\Hb$ and norm-increasing subset (see Figure~\ref{fig: hyperrectangle}) are key ingredients for understanding single-neuron linear networks.

\begin{proposition}[Properties of the hyperrectangle of $\Hb$ and the norm-increasing subset] \label{prop: hyperrectangle}
	Suppose we have a unique minimum $\wb^*$ for \eqref{eq: cost linear}. 
	Then the following statements hold.
	
	\begin{enumerate}
		\item The norm of a gradient flow increases if and only if it is inside the norm-increasing subset.
		\item The norm-increasing subset forms a $d$-dimensional ellipsoid, where axes of the ellipsoid are parallel to the eigenvectors of $\Hb$. 
		\item If a gradient flow initializes in the hyperrectangle of $\Hb$, it never escapes the hyperrectangle.
		\item The ellipsoid contains the hyperrectangle. In particular, their center points coincide at $\frac{1}{2}\wb^*$ and every vertex of the hyperrectangle lies on the boundary of the norm-increasing subset.
	\end{enumerate}
\end{proposition}
\begin{proof}
	\begin{enumerate}
		\item Consider the time derivative of $\| \wb(t) \|^2$.
		\begin{align*}
			\frac{d}{dt} \norm{\wb}^2 &= 2 \wb^T \dot{\wb} \\
			&= -2g(\wb)
		\end{align*}
		Therefore, the norm of the gradient flow $\wb$ increases if and only if $g(\wb)<0$, which means $\wb$ is inside the norm-increasing subset from Definition \ref{def: g(w)}. 
		
		\item Since $\nabla_\wb L = \Hb\wb - \qb = \Hb(\wb-\wb^*)$ by Proposition \ref{prop: stationary point linear},
		\begin{align*}
			g(\wb) &= -\wb^T \dot{\wb} \\
			&= \wb^T \Hb(\wb-\wb^*) \\
			&= (\wb-\frac12\wb^*)^T \Hb(\wb-\frac12\wb^*) -\frac14\wb^{*T}\Hb\wb^*.
		\end{align*}
		Therefore, $g(\wb)<0$ is an ellipsoid centered at $\frac12 \wb^*$, where its axes are parallel to eigenvectors of $\Hb$ by \citet{anton2003contemporary}.
		
		\item Since $\Hb\wb^*=\qb$, \eqref{eq: linear gradient flow} becomes
		\begin{align}\label{eq: linear gradient flow with w*}
			\dot{\wb} = -\Hb (\wb - \wb^*), \quad \wb(0) = \wb_0.
		\end{align}
		For $k$-th eigenvectors $\eb_k$ of $\Hb$, define $c_k^0 := \wb_0^T\eb_k$ and $c_k^* := \wb^{* T}\eb_k$.
		From \citet{zill2020advanced}, the closed form solution of \eqref{eq: linear gradient flow with w*} is given by
		\begin{align}
			\wb(t) &= e^{-\Hb t}(\wb_0 - \wb^*) + \wb^*
			\label{eq: linear gradient flow solution} \\
			&= \sum_{k=1}^{d} e^{-\lambda_k t}\eb_k \eb_k^T (\wb_0-\wb^*) + \wb^* \label{eq: linear gradient flow solution eigenvector version} \\
			&= \sum_{k=1}^{d}\bigg( e^{-\lambda_k t}\eb_k (c_k^0 - c_k^*) + c_k^* \eb_k \bigg) 
			\notag \\
			&= \sum_{k=1}^{d}\bigg( c_k^* - e^{-\lambda_k t}(c_k^* - c_k^0)   \bigg) \eb_k.
			\notag
		\end{align}
		Since $\wb_0$ is inside the hyperrectangle, we get $0\le c_k^0 \le c_k^*$. Then, $c_k^* - e^{-\lambda_k t}(c_k^* - c_k^0) \le c_k^*$ and $c_k^* - e^{-\lambda_k t}(c_k^* - c_k^0) = (1 - e^{-\lambda_k t})c_k^* +  e^{-\lambda_k t}c_k^0 \ge 0$ imply $0 \le \wb(t)^T\eb_k = c_k^* - e^{-\lambda_k t}(c_k^* - c_k^0) \le c_k^*$ for all $t \ge 0$. Therefore, the gradient flow never escapes the hyperrectangle.

		\item Since the hyperrectangle is convex, it is enough to show that the vertices of the hyperrectangle lie on the boundary of the ellipsoid.
		Let $\vb_j$ be one vertex of the hyperrectangle among $2^d$ vertices. Then, it can be represented by 
		\begin{align*}
			\vb_\jb = \sum_{k=1}^d j_k c_k^* \eb_k, \qquad j_k = 0 \text{ or } 1
		\end{align*}
		where $c_k^* = \wb^{*T}\eb_k$. Subsequently,
		\begin{align*}
			g(\vb_\jb) &= \vb_\jb^T \Hb (\vb_\jb - \wb^*) \\
			&= \left(\sum_{k=1}^d j_k c_k^* \eb_k\right)^T \left(\sum_{k=1}^d \lambda_k \eb_k \eb_k^T \right) \left(\sum_{k=1}^d (1-j_k) c_k^* \eb_k\right) \\
			&=\sum_{k=1}^d j_k (1-j_k) \lambda_k c_k^{*2} \\
			&= 0.
		\end{align*}
		Therefore, all vertices of the hyperrectangle lie on the boundary of the norm-increasing ellipsoid, thus the hyperrectangle is contained in the ellipsoid.
	\end{enumerate}
\end{proof}

We now rewrite the well-known properties of least-square linear regression in \citet{gunasekar2017implicit}.
\begin{theorem}[Norm increasing property and implicit bias of single-neuron linear networks]
    \label{thm: norm increasing linear}
	Consider \eqref{eq: linear gradient flow} and suppose $\wb_0 = \zerob$. Then,
	\begin{enumerate}
		\item The converged point is the global minimum with  the minimum Euclidean norm.
		\item The norm of the gradient flow $\norm{\wb}$ monotonically increases until it converges.
	\end{enumerate}
\end{theorem}
\begin{proof}
	\begin{enumerate}
		\item If $\Xb$ is full rank,  there is only one stationary point so that the statement trivially holds. Now, suppose $\Xb$ is not full rank, and recall the reduction principle (Theorem \ref{thm: reduction principle}). 
		Let $\overline{\wb}^*$ be the convergent point of the gradient flow. Since the gradient flow lies on $\zerob + \overline\Wc$, $\overline{\wb}^* \in \overline\Wc$.
		Let $\wb^*$ be another minimum of $L$ in $\Wc$. By the last statement of Theorem \ref{thm: reduction principle}, $\wb^* = \overline{\wb}^* + \vb$ for some $\vb\in$ker$(\varphi)$. Then,
		\begin{align*}
			\norm{\wb^*}^2 &= \norm{\overline{\wb}^* + \vb}^2 \\
			&= \norm{\overline{\wb}^*}^2 + \norm{\vb}^2 \\
			&\ge \norm{\overline{\wb}^*}^2.
		\end{align*}
		
		Therefore, the convergent point $\overline\wb^*$ is the minimum Euclidean norm stationary point.
		
		\item Since the gradient flow is contained in $\overline\Wc$, consider the hyperrectangle of $\Hb$ in $\overline\Wc$. Since $\zerob \in \overline\Wc$, the gradient flow initializes in the hyperrectangle. By Proposition \ref{prop: hyperrectangle}, the gradient flow never escapes the hyperrectangle, which is contained in the norm-increasing subset. Therefore, the norm of the gradient flow increases until it converges $\overline{\wb}^*$.
	\end{enumerate}
\end{proof}


\section{Reduction principle}
In this section, we explain why we can assume $\Xb$ to have full rank without loss of generality. %

\begin{figure}[ht!]
	\centering
	\includegraphics[width=0.6\textwidth]{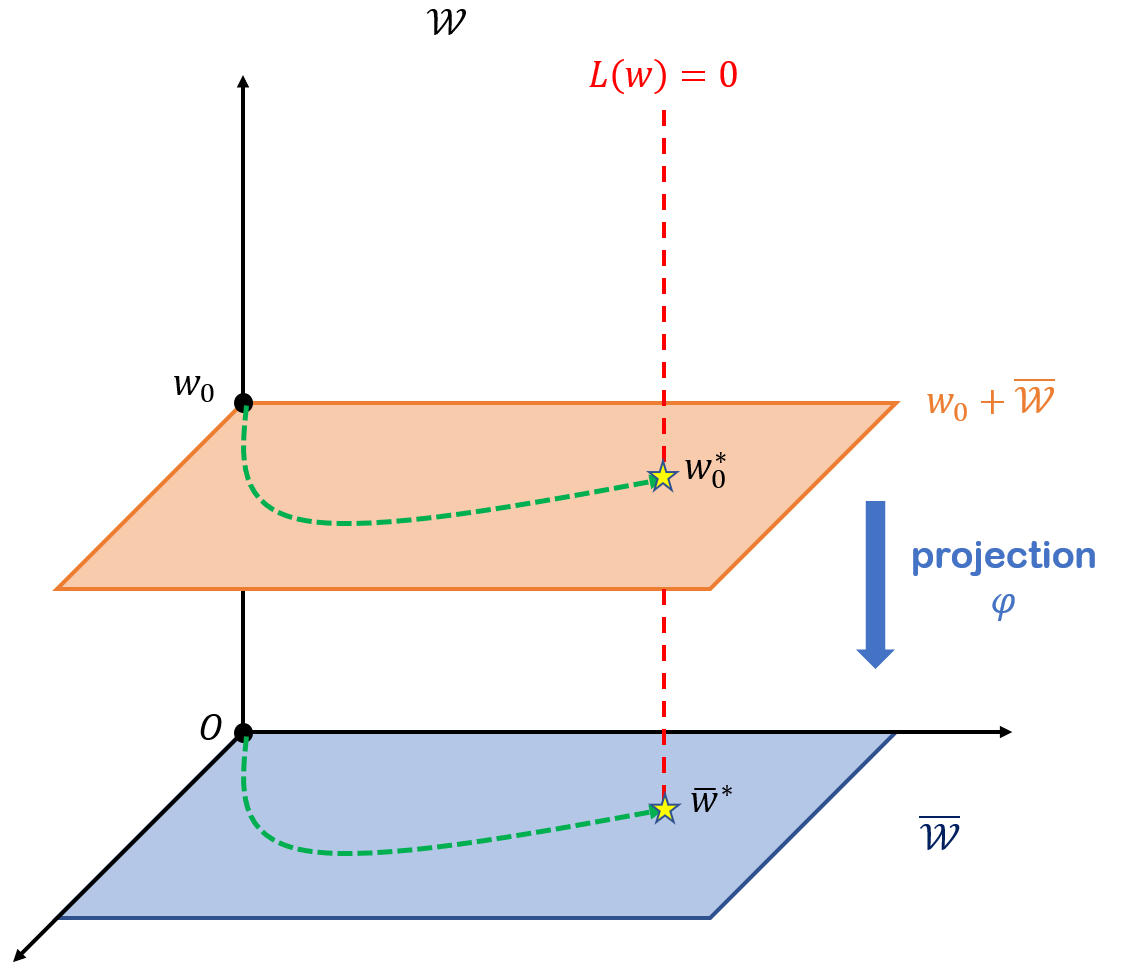}
	\caption{Reduction principle. 
		There are two gradient flows denoted by dashed green curves. Note that the trajectory of the gradient flows coincide in $\overline{\Wc}$ through the projection $\varphi$.
	}
	\label{fig: projection}
\end{figure}

\begin{theorem}[Reduction principle] \label{thm: reduction principle}
	In the last case of Proposition \ref{prop: stationary point linear}(i.e., $\Xb$ is not full rank), there exists a projection $\varphi:\Wc \rightarrow \Wc$ and 
	orthogonal decomposition $\Wc = \overline{\Wc} \oplus \textup{ker}(\varphi)$ such that 
	\begin{enumerate}
		\item for any $\overline{\wb} \in \overline{\Wc}$ and $\vb \in \textup{ker}(\varphi)$, $L(\wb) = L(\wb + \vb)$. 
		
		\item The gradient flow \eqref{eq: linear gradient flow} lies on $\wb_0 + \overline{\Wc}$.
		
		
		\item $L$ has a unique minimum $\overline{\wb}^*$ in $\overline\Wc$.
		
		\item For every minimum $\wb^*$ of $L$, $\varphi(\wb^*) = \overline{\wb}^* \in \overline{\Wc}$.
	\end{enumerate}
	In other words, the dynamics of the gradient flow $\wb(t)$ in $\Wc = \Rd^d$ can be considered in $\overline{\Wc}\cong \Rd^r$.
\end{theorem}
\begin{proof}
	Since $r<d$, the Hessian matrix $\Hb = \Xb \Xb^T$ has $r$ positive eigenvalues
	$$ \lambda_1 \ge \lambda_2 \ge \cdots \ge \lambda_r > \lambda_{r+1} = \cdots = \lambda_d = 0.
	$$
	Now consider spectral decomposition of $\Hb=\sum\limits_{k=1}^r \lambda_k \eb_k \eb_k^T$.
	Define a projection operator $\varphi:\Wc \rightarrow \overline{\Wc}$ by
	\begin{align*}
		\varphi~:~&\Wc \quad\longrightarrow \qquad\quad \Wc \\
		&\, \wb \quad\longmapsto\quad \sum_{k=1}^r (\wb^T\eb_k)\eb_k^T.
	\end{align*}
	Then it is easily checked that $\varphi$ is onto. Therefore, we get $\overline{\Wc}:=\varphi(\Wc) \cong \Rd^r$ and a direct sum decomposition $\Wc = \overline{\Wc} \oplus \textup{ker}(\varphi)$.
	Now we are ready to prove the statements of the theorem.
	\begin{enumerate}
		\item For any $\vb \in \textup{ker}(\varphi)$, 
		\begin{align*}
			L(\wb+\vb) &= \frac12 \sum_{i=1}^n ((\wb+\vb)^T\xb_i - y_i)^2 \\
			&= \frac12 \sum_{i=1}^n (\wb^T\xb_i + \vb^T\xb_i - y_i)^2 \\
			&= \frac12 \sum_{i=1}^n (\wb^T\xb_i - y_i)^2 \\
			&= L(\wb).
		\end{align*}
		
		\item Let $\wb_\infty$ be a stationary point of $L$. Then, $\Hb\wb_\infty = \qb$ implies $\dot{\wb}=-\nabla L(\wb) = -\Hb\wb+\qb =-\Hb(\wb - \wb_\infty) \in \textup{range}(H)$. Thus $\varphi(\dot{\wb})=\dot{\wb}$. Then, 
		\begin{align*} 
			\varphi(\wb - \wb_0) &= \varphi(\wb(t)-\wb(0)) \\
			&= \varphi (\int_0^t \dot{\wb}(s)ds) \\
			&= \int_0^t \varphi (\dot{\wb}(s))ds \\
			&=\int_0^t \dot{\wb}(s)ds \\
			&= \wb(t) - \wb(0) \\
			&= \wb - \wb_0.
		\end{align*}
		
		Therefore, $(\wb(t) - \wb_0)$ is in $\overline{\Wc}$, which means the gradient flow lies on $\wb_0 + \overline{\Wc}$.
		
		
		\item From the construction, $\overline{\Wc}=$ range$(\Hb)$. Therefore, Proposition \ref{prop: stationary point linear} gives the unique stationary point $\overline{\wb}^*$ in $\overline{\Wc}$.
		
		\item Let $\wb^*$ be a minimum of $L$. From orthogonal decomposition $\Wc = \overline{\Wc} \oplus \textup{ker}(\varphi)$, $\wb^*=\varphi(\wb^*) + (\wb^* - \varphi(\wb^*))$. By the first statement, $L(\wb^*) = L(\varphi(\wb^*) + (\wb^* - \varphi(\wb^*))) = L(\varphi(\wb^*))$. Therefore, $\varphi(\wb^*)$ is a minimum in $\overline{\Wc}$. Since $\overline{\Wc}$ has a unique minimum $\overline{\wb}^*$, we conclude $\varphi(\wb^*)=\overline{\wb}^*$.
	\end{enumerate}
	
	In particular, by (ii), we can say that the geometry of the gradient flow can be reduced to $\overline{\Wc}$. See Figure \ref{fig: projection}.
\end{proof}

\begin{corollary}[Reduction principle for ReLU networks] \label{cor: ReLU reduction principle}
	Consider a single-neuron ReLU network with $\{\xb_i\}_{i=1}^n$. Suppose $\Xb$ has rank $r<d$. Then, the dynamics of gradient flows in $\Wc=\Rd^d$ can be reduced into $\overline{\Wc} \cong \Rd^r$, by the projection $\varphi$ in Theorem \ref{thm: reduction principle}. In other words, for overparameterized setting, we can reduce it to critically determined case $\overline{d}=r$.
\end{corollary}
\begin{proof}
	The proof is exactly same with the proof of Theorem \ref{thm: reduction principle}. For the last statement, overparameterized setting implies rank$(X)=r \le n <d$. 
\end{proof}
{
Thanks to this corollary, we can assume $\Xb$ has full rank without loss of generality (\ref{asm: underparameterization}).
}

\section{Proofs of Lemmas, Propositions and Theorems}
\subsection{Partitions and support vectors of single-neuron ReLU networks}

\begin{lemma} \label{lem: partition}
	Let $P_+$ and $P_-$ be two adjoined partitions, where the common boundary is determined by a data $\xb^*$. Let $P_+$ be a partition where $\xb^*$ is activated. Then,
	$$
	\{\xb \in \Xb ~|~ \xb \sim P_+\} = \{\xb \in \Xb ~|~ \xb \sim P_-\} \cup \{\xb^*\}.
	$$
\end{lemma}
\begin{proof}
\vspace{-12pt}
	Since $P_+$ and $P_-$ are adjoined, their activation patterns are different in exactly one data, which is $\xb^*$. Note that $\xb^*$ is activated in $P_+$ and deactivated in $P_-$. Therefore, $ \{\xb \in \Xb ~|~ \xb \sim P_+\} = \{\xb \in \Xb ~|~ \xb \sim P_-\} \cup \{\xb^*\}$. 
\end{proof}

\begin{lemma} \label{prop: number of partitions}
	Consider $n$ training data $\{\xb_i\}_{i=1}^n$ in $\mathbb{R}^d$ under \ref{asm: input}. Then the number of partitions in the parameter space is at most
	\begin{align} \label{eq: number of partitions}
		2\sum_{k=0}^{d-1} \begin{pmatrix}
			n-1 \\ k
		\end{pmatrix}.
	\end{align}
\end{lemma}
\begin{proof}[Proof of Lemma~\ref{prop: number of partitions}]
    Zaslavsky's Theorem tells  that a general $n$ hyperplanes in $\Rd^d$ generate at most $ \sum\limits_{k=0}^{d} \begin{pmatrix} n \\ k \end{pmatrix} $ partitions. From here, we induce the case for central hyperplanes. Consider a hyperplane $H$ in $\Wc$, and define two parallel hyperplanes $H_+$ and $H_-$ such that $H$ is between them. Every partition is on one side of $H$ or another, so it either intersects with $H_+$ or $H_-$. Since either of $H_+$ or $H_-$ is $(d-1)$ dimensional, it is divided by other $n-1$ hyperplanes into at most $ \sum\limits_{k=0}^{d-1} \begin{pmatrix} n-1 \\ k \end{pmatrix} $ partitions. Therefore, the total number of partitions in $\Wc$ is at most $ 2\sum\limits_{k=0}^{d-1} \begin{pmatrix} n-1 \\ k \end{pmatrix} $.
\end{proof}

\begin{proposition}[Basic properties of single-neuron ReLU networks] 
\label{prop: basic properties}
	For single-neuron ReLU networks under \ref{asm: input}, the parameter space has the following properties.
	\vspace{-3mm}
	\begin{enumerate}\setlength{\itemsep}{-1mm}
		\item Every partition is convex and unbounded. In particular, if $\wb \in P$ for a partition $P$, then $\alpha \wb \in P$ for any scalar $\alpha>0$ (i.e., $P$ is conic). Similarly, $\Hb(\wb)$ and $\qb(\wb)$ in \eqref{eq: Hq} are invariant under multiplication by a positive constant.
		\item If there is a nonpositive label $y_i \le 0$, a function $f_i(\wb)=\frac12 ([\wb^T\xb_i]_+ - y_i)^2$ is convex with respect to $\wb$.
		In general, if there are $m$ positive labels among $n$ data, the number of partitions that contain a local minimum is at most $ 2\sum\limits_{i=0}^{d-1} \begin{pmatrix}
			m-1 \\ i
		\end{pmatrix}+1$.
	\end{enumerate}
\end{proposition}
{Proposition~\ref{prop: basic properties} shows that negative labels do not affect to number of partitions which has a local minimum. This is the reason why we assume \ref{asm: label}.}

\begin{proof}[Proof of Proposition~\ref{prop: basic properties}]
    ~\vspace*{-12pt}
	\begin{enumerate}[itemindent=10mm, leftmargin=0mm]
		\item Let $P$ be a linearly partitioned region with an activation pattern $\mm_{\{\wb^T\Xb\}>0}$. Then for $\wb_1, \wb_2 \in P$ and $0 \le \lambda \le 1$, every interpolation point $\lambda\wb_1 + (1-\lambda)\wb_2$ keeps the value $\mm_{\{\wb^T\Xb>0\}}$ and thus it is contained in $P$. Similarly, for any positive scalar $\alpha>0$, $\alpha \wb$ keeps the value $\mm_{\{\wb^T\Xb\}>0}$, thus it is in $P$. Therefore, each linearly partitioned region is convex and unbounded. Finally, from \eqref{eq: Hq}, we can check that $\Hb(\alpha\wb)=\Hb(\wb)$ and $\qb(\alpha\wb)=\qb(\wb)$ for any $\alpha>0$.
		
		\item Note that ReLU is convex and $\wb^T\xb$ is affine with respect to $\wb$, thus $[\wb^T\xb]_+$ is convex again. For $y \le 0$, a function $([t]_+ -y)^2$ is non decreasing convex with respect to $t \in \mathbb{R}$. Finally, by using the fact that the composition of convex and convex nondecreasing function is again convex, we conclude that $f_i(\wb)=\frac12 ([\wb^T\xb_i]_+ - y_i)^2$ is convex if $y_i \le 0$.	
		
		Now, suppose we have $m$ positive labels among $n$ training data, i.e., $y_i>0$ for $1 \le i \le m$ and $y_i \le 0$ for $m+1 \le i \le n$.
		Consider $\sum\limits_{i=1}^m f_i(\wb)$ with partitions generated by $\{\xb_i\}_{i=1}^m$. Since it is convex on each partition, the number of partitions contain a local minimum is at most $ 2\sum\limits_{i=0}^{d-1} \begin{pmatrix} m-1 \\ i \end{pmatrix}$ by Lemma \ref{prop: number of partitions}.
		For the last $n-m$ data, by the above statement 1., we know $\sum\limits_{i=m+1}^n f_i(\wb)$ is globally convex. Considering the sum of a convex function and a piecewise convex function, the number of partitions contain a local minimum increases at most $1$.
		Therefore, the number of partitions that contain a local minimum is at most $ 2\sum\limits_{i=0}^{d-1} \begin{pmatrix} m-1 \\ i \end{pmatrix}+1$.
	\end{enumerate}
	\vspace{-1cm}
\end{proof}

\begin{proof}[Proof of Proposition \ref{prop: activation}]
    We prove $\Rightarrow)$ part first.
    If a gradient flow $\wb(t)$ deactivates a data $\xb$ at $t=s$, then $\wb(t)$ is on the activation boundary of $\xb$ at $t=s$.
    This implies the first statement $\wb(s)^T\xb=0$.
    For the second statement, from the definition of deactivation, there exists $\delta>0$ such that for $s<t<s+\delta$, $\wb(t)^T\xb<0$. Then, we get
    $(\wb(t)-\wb(s))^T\xb <0$ for $s<t<s+\delta$, which implies
    \begin{align*}
        -\nabla L(\wb(s))^T\xb 
        = (\frac{d\wb}{dt}\Big|_{t=s})^T \xb
        = (\lim\limits_{t \rightarrow s^+} \frac{\wb(t)-\wb(s)}{t-s})^T\xb
        \le 0.
    \end{align*}
    Therefore, we get $\nabla L(\wb(s))^T\xb \ge 0$. \\
    For $\Leftarrow)$ part, note that $\nabla L(\wb(s))^T\xb \ge 0$ implies $(\frac{d\wb}{dt}|_{t=s})^T \xb \le 0$, thus there exists $\delta>0$ such that  $(\wb(t)-\wb(s))^T\xb \le 0$ for $s<t<s+\delta$. Since $\xb^T\wb(s)=0$, we conclude that $\wb(t)<0$ for $s<t<s+\delta$,
    which means $\wb(t)$ deactivates $\xb$ at $t=s$.
\end{proof}

\begin{figure}[!bt]
	\centering
	\includegraphics[width=0.7\textwidth]{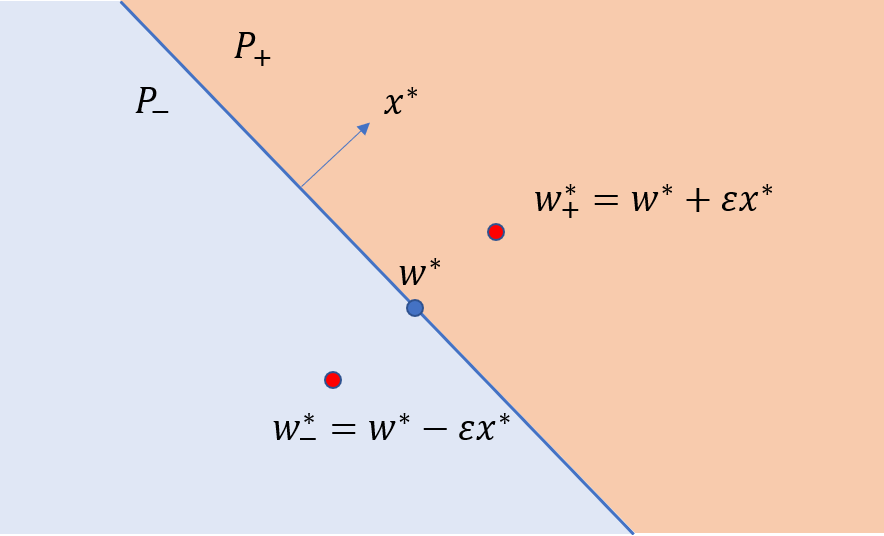}
	\caption{Proof of `Not on boundary lemma'(Lemma \ref{lem: not on boundary lemma}). A local minimum $\wb^*$ is assumed to lie on an activation boundary of $\xb^*$. Note that $\xb^*$ is activated in $P_+$, and deactivated in $P_-$. }
	\label{fig: not on boundary}
\end{figure}
\begin{proof}[Proof of Lemma~\ref{lem: not on boundary lemma}]
	We will prove this by contradiction.
	Suppose a local minimum $\wb^*$ is on an activation boundary, and let $\xb^*$ be the corresponding data of the activation boundary with the label $y^*>0$. Consider two perturbed vectors $\wb_{\pm} = \wb^* \pm \varepsilon \bfx^*$ with sufficiently small $\varepsilon>0$. 
	Let $P_+$ and $P_-$ be adjoined linearly partitioned regions, where each partition contains $\wb_{+}$ and $\wb_-$, respectively (see Figure~\ref{fig: not on boundary}).
	Then, we have
	\begin{align*}
		\nabla L(\wb_+) &= \sum\limits_{\xb_i \sim P_+}(\wb_+^T\bfx_i - y_i)\bfx_i, \\
		\nabla L(\wb_-) &= \sum\limits_{\xb_i \sim P_-}(\wb_-^T\bfx_i - y_i)\bfx_i.
	\end{align*} 
	Furthermore, by Lemma~\ref{lem: partition}, we have
	\begin{align*}
		\nabla L(\wb_+) &= \sum\limits_{\xb_i \sim P_-}(\wb_+^T\bfx_i - y_i)\bfx_i + (\wb_+^T\bfx^* - y^*)\bfx^* .
	\end{align*} 
	Since $L(\wb)$ is convex on each linearly partitioned region, following properties hold within each partition :
	\begin{align*}
		0 \ge L(\wb^*) - L(\wb_+) &\ge \nabla L(\wb_+)^T (\wb^* - \wb_+), \\
		0 \ge L(\wb^*) - L(\wb_-) &\ge \nabla L(\wb_-)^T (\wb^* - \wb_-)   
	\end{align*}
	which imply 
	$\nabla L(\wb_+)^T\bfx^* \ge 0$ and $\nabla L(\wb_-)^T\bfx^* \le 0$, respectively.
	Therefore, we get $(\nabla L(\wb_+) - \nabla L(\wb_-))^T\xb^* \ge 0$.
	Taking $\varepsilon \rightarrow 0^+$, $\wb_+$ and $\wb_-$ coincide to $\wb^*$ and 
	\begin{align*}
	    \lim_{\varepsilon \rightarrow 0^+} (\nabla L(\wb_+) - \nabla L(\wb_-))
	    &= (\wb^{*T}\bfx^*-y^*)\xb^*  \\
	    &= -y^*\xb^*
	\end{align*}
	since $\wb^{*T}\bfx^*=0$. Therefore, we get
	\begin{align*}
		0 \le [\nabla L(\wb_+) - \nabla L(\wb_-)]^T\xb^* \; \xrightarrow{\varepsilon \rightarrow 0^+} \; -y^*\norm{\xb^*}^2 < 0,
	\end{align*}
	which is a contradiction. This completes the proof.
\end{proof}

\begin{proof}[Proof of Proposition~\ref{prop: ReLU solution}]
	Since $L$ is convex on each partition, we can apply Proposition \ref{prop: stationary point linear}.
    Thus we need to check whether $\wb_P^*$ is in $P$, which is \eqref{eq: partition condition}.	
	If $r_P < d$, by Proposition \ref{prop: stationary point linear} again, the set of local minimum forms a $(d-r_P)$-dimensional connected linear manifold.
\end{proof}

\begin{proof}[Proof of Theorem~\ref{thm: ReLU get smaller loss}]
	From \eqref{eq: cost linear} and \eqref{eq: cost ReLU}, let $L_l$ and $L_r$ be the loss functions of single-neuron linear and ReLU networks, which are defined by
	\begin{align*}
		L_l(\wb) &:= \frac12 \sum_{i=1}^n (\wb^T\bfx_i - y_i)^2  \\
		L_r(\wb) &:= \frac12 \sum_{i=1}^n ([\wb^T\bfx_i]_+ - y_i)^2.
	\end{align*}
	Let $\wb_l^*$ and $\wb_r^*$ be the global minimum of $L_l(\wb)$ and $L_r(\wb)$, respectively. Let $S$ be the set of support vectors of $\wb_r^*$. Then,
	\begin{align*}
		L_l(\wb_l^*) &= \frac12 \sum (\wb_l^{* T}\bfx_i - y_i)^2 \\
		&=  \frac12 \sum_{\bfx_i \in S} (\wb_l^{* T}\bfx_i - y_i)^2 + \frac12 \sum_{\bfx_i \not\in S} (\wb_l^{* T}\bfx_i-y_i)^2 \\
		&\ge \frac12 \sum_{\bfx_i \in S} (\wb_l^{* T}\bfx_i-y_i)^2 + \frac12 \sum_{\bfx_i \not\in S} y_i^2 \\
		&= L_r(\wb_l^*) \\
		&\ge L_r(\wb_r^*)
	\end{align*}
	where the first inequality comes from $y_i > 0$ and $\wb_l^{*T}\xb_i \le 0$ for $\xb_i \not \in S$, and the last inequality comes from the fact that $\wb_r^*$ is the global minimum of $L_r$. Therefore, the global minimum of single-neuron ReLU networks achieve equal or smaller loss than global minimum of single-neuron linear networks.
\end{proof}

\begin{proof} [Proof of Theorem~\ref{thm: more support vectors}]
	It is almost same with the proof of Theorem \ref{thm: ReLU get smaller loss}.
	\begin{align*}
		L(\wb_2^*) 
		&= \min_\wb \sum_{\xb_b \in S_2} (\wb^T\bfx_b - y_b)^2 + \sum_{\xb_c \in S_2^c} (0-y_c)^2 \\
		&= \min_\wb \sum_{\xb_b \in S_2} (\wb^T\bfx_b - y_b)^2 + \sum_{\xb_c \in S_1-S_2} (0-y_c)^2 + \sum_{\xb_c \in S_1^c} (0-y_c)^2 \\
		&= \min_{\stackrel{\wb^T\bfx_c \le 0}{\text{ for } \xb_c \in S_1-S_2}} \;\; \sum_{\xb_b \in S_1} (\wb^T\bfx_b - y_b)^2 + \sum_{\xb_c \in S_1^c} (0-y_c)^2 \\
		&\ge  \min_\wb \sum_{\xb_b \in S_1} (\wb^T\bfx_b - y_b)^2 + \sum_{\xb_c \in S_1^c} (0-y_c)^2 \\
		&= L(\wb_1^*).
	\end{align*}
\end{proof}

\subsection{Weight initialization} 
\begin{proof}[Proof of Proposition~\ref{prop: effect of norm}]
    Recall the positive homogeneity described in Proposition \ref{prop: basic properties} : $\alpha \wb \in P$ for $\alpha >0$.
    From \eqref{eq: Hq}, we get
	\begin{align*}
	    -\nabla L(\alpha \wb)^T \xb_j &= \xb_j^T (-\Hb\alpha \wb +\qb) \\
	    &= -\alpha \xb_j^T \Hb \wb + \xb_j^T\qb.
	\end{align*}
	Therefore, $-\nabla L(\alpha \wb)^T \xb_j \ge 0 $ if and only if 
	$$ \alpha \le \frac{\xb_j^T\qb}{\xb_j^T \Hb \wb}
	= \alpha_j^*.
	$$
	This means that if $0<\alpha \le \min\limits_j \alpha_j^*$, $-\nabla L(\alpha\wb)^T\xb_j > 0 $ for all $j$. Similarly, if $\alpha > \max\limits_j \alpha_j^*$, $-\nabla L(\alpha\wb)^T\xb_j < 0 $ for all $j$. Thus for a gradient flow $\wb(t)$ with large enough $\norm{\wb}$, {by Proposition \ref{prop: activation},} it moves to a direction that deactivates any data. Similarly, for small enough $\norm{\wb_0}$, it moves to a direction that activates all data. See Figure \ref{fig: effect of norm}.
\end{proof}


\begin{proposition}[Well-definedness of gradient flows] \label{prop: well-definedness}
{
    Consider a single-neuron ReLU network. In the extended sense, the gradient flow defined by \eqref{eq: Hq} is well-defined and has a unique solution for a given $\wb_0$.}
\end{proposition}
\begin{proof}
{
    Note that the differential equation has no solution in a narrow sense since the right hand side is discontinuous. However, by Carath\'eory's existence theorem \cite{coddington1955theory},
    the solution exists in the extended sense, where \eqref{eq: ReLU gradient flow} holds for all $t$ except on a set of Lebesgue-measure zero.
    Moreover, the solution is unique in each partition, where the right hand side is Lipschitz continuous. Therefore, we have a unique gradient flow defined by \eqref{eq: Hq}.
    }
\end{proof}

{Now we define {\em all-activated partition} as the partition that activates all data, which  
includes the 1st quadrant $\{\wb | \wb > \zerob \}$. i.e., the set of support vectors of all-activated partition is whole dataset.}
Then, every gradient flow initialized with sufficiently small norm and has positive gradient enters to the all-activated partition in short time, as explained in the following Lemma \ref{lem: all activated}.

\begin{lemma} \label{lem: all activated}
    {Consider a single-neuron ReLU network under \ref{asm: input}, \ref{asm: label}, and \ref{asm: underparameterization}. 
	Then, there exists $\delta>0$ such that if $\norm{\wb_0}<\delta$ and $\qb(\wb_0)>\zerob$, any gradient flow initialized at $\wb_0$ enters to the all-activated partition.}
\end{lemma}
\begin{proof}
    For $\alpha>0$, the gradient at $\alpha\wb$ is given by
    \begin{align*}
        -\nabla L(\alpha\wb) = - \Hb(\alpha\wb) \alpha\wb + \qb(\alpha\wb) = -\alpha \Hb(\wb) \wb + \qb( \wb),
    \end{align*}
    thus it converges to $\qb(\wb)$ as $\alpha \rightarrow 0^+$.
    In this case, the first order linear approximation of the gradient flow is given by
    \begin{align*}
        \wb(t) \approx \wb_0 + \qb \:t.
    \end{align*}
    Moreover, since $\qb$ has positive entries, $\qb^T \xb_i > 0$ for all $i$.
    Therefore, the gradient flow moves to the 1st quadrant, where every data is  activated.
    After activating a data $\xb_j$, $\qb$ would be added by $y_j\xb_j$, so remain as a vector with positive entries.
    {Hence, we can find $\delta>0$ such that any gradient flow initialized at $\norm{\wb_0}<\delta$ enters to the all-activated partition.} 
\end{proof}

{
\begin{proof}[Proof of Theorem~\ref{thm: ReLU no deactivation}]
    Consider a situation that the gradient flow $\wb(t)$ initialized at $\wb_0$ is in the $l$-th partition ($\wb_0$ is in the 0-th partition). Let $0< t_1 < t_2 < \cdots < t_l$ be the time when $\wb(t)$ crosses an activation boundary. Then $\wb(t)$is contained in one partition for $t_i \le t \le t_{i+1}$. Define $\Hb_i:=\Hb(w(t)), t\in (t_i,t_{i+1})$. 
    Since $\wb_{GM}^*$ satisfies $\xb_j^T \wb_{GM}^* = y_j$, 
    the gradient flow $\wb(t)$ is represented by
    \begin{align*}
        \wb(t_1) - \wb_{GM}^* &= e^{-\Hb_0 t_1} (\wb_0-\wb_{GM}^*) \\
        \wb(t_2) - \wb_{GM}^* &= e^{-\Hb_1 (t_2 -t_1)} (\wb(t_1)-\wb_{GM}^*) \\
        & \; \vdots \\
        \wb(t_l) - \wb_{GM}^* &= e^{-\Hb_{l-1} (t_l -t_{l-1})} (\wb(t_{l-1})-\wb_{GM}^*)
    \end{align*}
    Therefore,
    \begin{align*}
        \wb(t_l) - \wb_{GM}^* = &e^{-\Hb_{l-1} (t_l -t_{l-1})} e^{-\Hb_{1-2} (t_{l-1} -t_{l-2})} \\
        & \cdots e^{-\Hb_{0} t_1 } (\wb_0-\wb_{GM}^*).
    \end{align*}
    To investigate $h_j(t) :=\xb_j^T \wb(t)$, consider
    \begin{align*}
        \norm{\xb_j^T (\wb(t_l) - \wb_{GM}^* )}
        &= \norm{ \xb_j^T e^{-\Hb_{l-1} (t_l -t_{l-1})} \cdots e^{-\Hb_{0} t_1} (\wb_0-\wb_{GM}^*) } \\
        &\le \norm{ \xb_j^T e^{-\Hb_{l-1} (t_l -t_{l-1})} \cdots e^{-\Hb_{0} t_1}} \cdot \norm{ \wb_0 - \wb_{GM}^*}
    \end{align*}
    The first term of the last line can be approximated by a simple exponential term.
    \begin{align*}
        \norm{ \xb_j^T e^{-\Hb_{l-1} (t_l -t_{l-1})} \cdots e^{-\Hb_{0} t_1}}
        & \le \norm{ \xb_j^T e^{-\Hb_{l-1} (t_l -t_{l-1})}\cdots e^{-\Hb_{1} (t_2 -t_{1})}  } e^{-\lambda^+_{min} (\Hb_{0}) t_1} \\
        & \le \norm{ \xb_j^T e^{-\Hb_{l-1} (t_l -t_{l-1})} \cdots e^{-\Hb_{1} (t_2 -t_{1})}} e^{-m t_1} \\
        & \le \cdots \\
        & \le \norm{\xb_j } e^{-m(t_l - t_{l-1})} \cdots e^{-mt_1} \\
        &\le \norm{\xb_j } e^{-m t_l}
    \end{align*}
    where $m=\min\limits_{0 \: \le \: k \: \le \: l-1} \lambda_{min}^+(\Hb_k)$.
    We modify the conventional inequality $\norm{e^{-\Ab\Ab^T}\xb} \le e^{-\lambda_{min}(\Ab\Ab^T)} \norm{\xb}$ to the inequality that $\norm{e^{-\Ab\Ab^T}\xb} \le e^{-\lambda^+_{min}(\Ab\Ab^T)} \norm{\xb}$ when $\xb$ is in col$(\Ab)$, which is induced by the reduction principle(Theorem \ref{thm: reduction principle}). Indeed, the inequality proposed above holds since $\xb_j^T e^{-\Hb_{l-1} (t_l -t_{l-1})} \cdots e^{-\Hb_{l'}(t_{l'+1}-t_{l'}) }$ is contained in $\text{col} (\Hb_{l'-1})$. Then,
    \begin{align*}
        | \xb_j^T (\wb(t_l) - \wb_{GM}^* )|
        &\le \norm{\xb_j} \cdot \norm{\wb_0-\wb_{GM}^*} e^{-mt_l}\\
        &< y_j e^{-mt_l} \\
        &<y_j.
    \end{align*}
    Therefore, we get
    \begin{align*}
        h_j(t_l) &= y_j - \xb_j^T (\wb(t_l) - \wb_{GM}^* ) >0.
    \end{align*}
    This shows that $\xb_j$ is activated for $0\le t \le t_l$. 
    Since $l$ is arbitrary, we show that $h_j(t)>0$ for $t\ge0$.
    


	Now we prove the last statement of the theorem. If \eqref{eq: no deactivation condition} holds for all $j=1,...,n$, there is no deactivation on the gradient flow. Therefore, it coincides with the gradient flow of a single-neuron linear network with the same initial point $\wb_0$, which converges to the global minimum by Proposition \ref{prop: stationary point linear}.
\end{proof}
}
\begin{proof}[Proof of Theorem~\ref{thm: gradient flow does not converge to a bad minimum}]
	By Theorem \ref{thm: ReLU no deactivation}, it is enough to show that there is some $\xb_j \in S^c$ satisfying \eqref{eq: no deactivation condition}. If \eqref{eq: condition for not bad minimum} holds, there exists some $\xb_j \in S^c$ which is always activated on the gradient flow $\wb(t)$. Therefore, the gradient flow cannot converge to $\wb_{loc}^*$. 
\end{proof}

\begin{proof}[Proof of Proposition~\ref{prop: zero initialization unique}]
    By Lemma \ref{lem: all activated}, we can assume $\wb_0 $ has sufficiently small norm, and is in the all-activated partition $P_0$.
    Suppose $\wb(t)$ converges to $\wb^*$ and $\wb(t)$ be contained in partitions $P_0, P_1, \dots P_l$ in time $(0,t_1), (t_1,t_2), \dots, (t_l,\infty)$.
    Then, $\wb(0) \in P_0$, which is all-activated partition, and $\wb^* \in P_l$.

    Note that $\wb^*$ is a fixed point.
    Especially, since $\Hb$ is positive definite, $\wb^*$ is an attractor in its neighborhood.
    Hence, we have $\varepsilon>0$, such that if a gradient flow pass $B_\varepsilon(\wb^*)$, then it converges to $\wb^*$.

    On the other hand, by the convergence of $\wb(t)$, we have $T$ such that $\wb(T) \in B_{\varepsilon/2}(\wb^*)$.
    Due to the dependency on the initial condition of ODE \cite{coddington1955theory}, we have $\delta_1$ such that if a gradient flow was contained in $B_{\delta_1}(\wb(t_l)) \cap P_l$, then it is contained in $B_{\varepsilon/2}(\wb(T))$ after time $T-t_l$.
    Thus, it would be contained in $B_\varepsilon(\wb^*)$ and converge to $\wb^*$.

    Next, consider a gradient flow that crosses from $P_{l-1}$ to $P_{l}$.
    The direction of $\dot{\wb}(t_l) = -\nabla L(\wb(t_l))$ is from $P_{l-1}$ to $P_l$.
    By the continuity of $-\nabla L$ on $P_{l-1}$, the direction of $-\nabla L$ is from $P_{l-1}$ to $P_l$ in $B_{\delta_1}(\wb(t_l))$ for some $\varepsilon'_1 < \delta_1$.
    Again by dependence on initial condition, there exists $\varepsilon_1<\varepsilon'_1$ such that a gradient flow in $B_{\varepsilon_1}(\wb(t_l))\cap P_{l-1}$ passes $B_{\delta_1}(\wb(t_l))\cap P_{l}$.
    
    And then, consider a gradient flow in $P_{l-1}$.
    We have $\delta_2$ such that if a gradient flow was contained in $B_{\delta_2}(\wb(t_{l-1})) \cap P_{l-1}$, then it is contained in $B_{\varepsilon_1}(\wb(t_l)) \cap P_{l-1}$ after time $t_l-t_{l-1}$.
    
    Successively, considering gradient flows that cross partitions and in partitions, we find $\varepsilon_2, \delta_2, \dots, $ and $\delta_{l+1}$ such that $B_{\delta_{l+1}}(\wb(0))\cap P_{0}$ converges to $\wb^*$.
\end{proof}

\subsection{Norm increasing property}
\begin{figure}[ht!]
	\centering
	\includegraphics[width=0.6\linewidth]{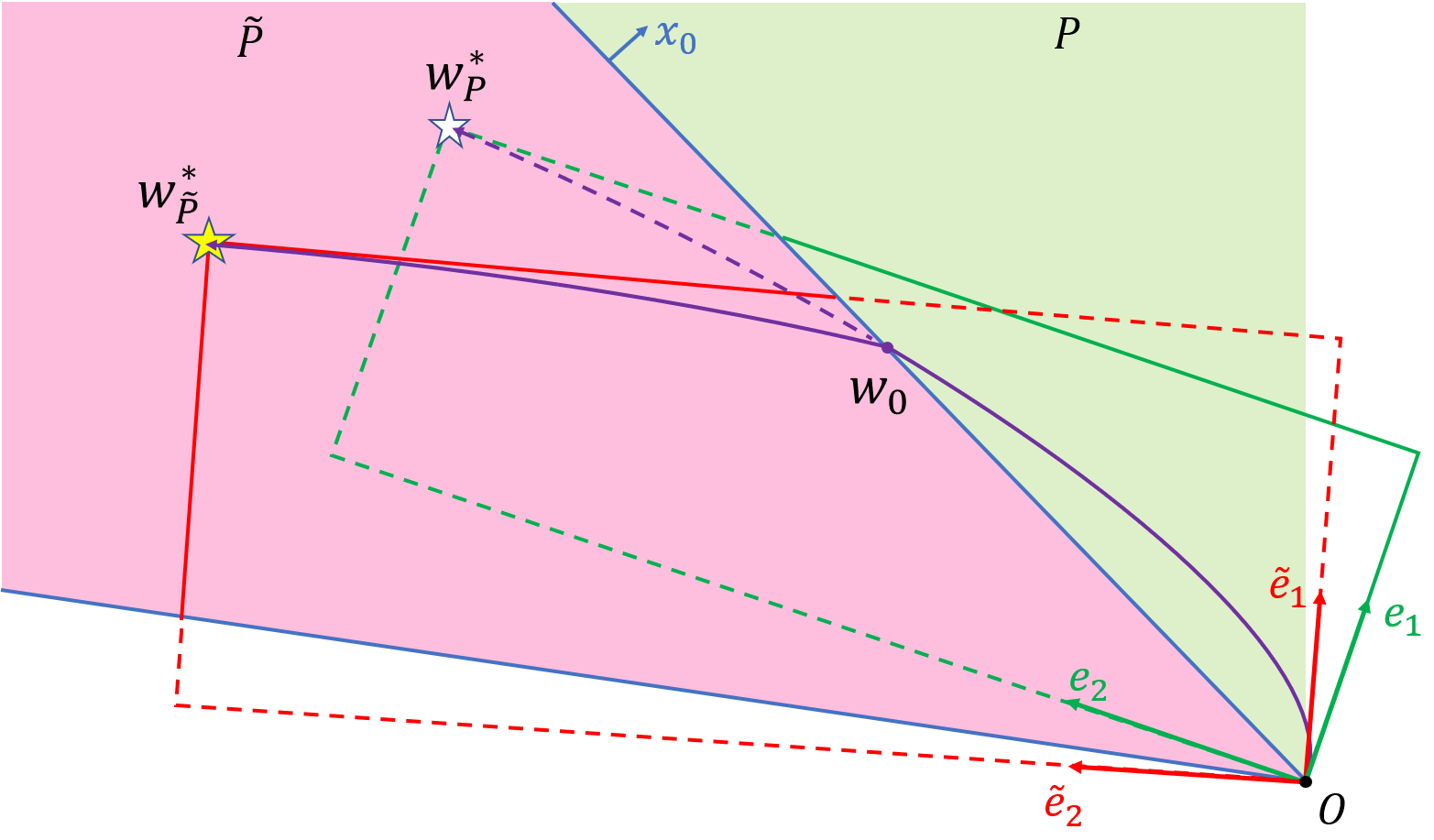}
	\caption{Norm increasing theorem for single-neuron ReLU networks. 
	The gradient flow in $P$ (light green region) moves to another partition $\tilde{P}$ (pink region) at $\wb_0$, deactivating (or activating) a data $\xb_0$. The associated hyperrectangles are drawn by solid and dashed lines, with eigenvectors. If the assumptions in Theorem \ref{thm: norm increasing} hold,
	the gradient flow $\wb$ (purple curve) is in the associate hyperrectangle in each partition, which means $\norm{\wb(t)}$ strictly increases.}
	\label{fig: norm increasing theorem}
\end{figure}

\begin{theorem}[Norm increasing property of single-neuron ReLU networks] \label{thm: norm increasing}
	Consider a single-neuron ReLU network under \ref{asm: input}, \ref{asm: label} and \ref{asm: underparameterization}. 
	Consider a gradient flow $\wb(t)$ moves from $P$ to $\tilde{P}$ across an activation boundary determined by $\xb_0$ at $\wb_0$ (either activating or deactivating). 
	For $\Hb:= \sum_{\xb_i \sim P} \xb_i \xb_i^T$ and  $\tilde{\Hb}:=\sum_{\xb_i \sim \tilde{P}} \xb_i \xb_i^T$, let $\{\lambda_k, \eb_k\}$ and $\{\tilde{\lambda_k}, \tilde{\eb_k}\}$ be the eigenvalues and eigenvectors of $\Hb$ and $\tilde{\Hb}$, respectively.
	Let $\wb_0=\sum_{k=1}^d c_k \eb_k = \sum_{k=1}^d \tilde{c_k} \tilde{\eb_k}$ has the coordinates by eigenvectors of $\Hb$ and $\tilde{\Hb}$. Similarly, $\wb_P^*=\sum c_k^* \eb_k$ and $\wb_{\tilde{P}}^* = \sum \tilde{c_k}^* \tilde{\eb_k}$ are the coordinates of the virtual minimum of $P$ and $\tilde{P}$. Let $\Delta \eb_k := \tilde{\eb_k} - \eb_k$ be the difference of $k$-th eigenvectors of $\tilde{\Hb}$ and $\Hb$. Similarly, $\Delta \wb^* := \wb_{\tilde{P}} ^* - \wb_P ^*$ is the difference of virtual minima of $\tilde{P}$ and $P$.
	Now assume the following conditions. 
	\vspace{-6pt}
	\begin{enumerate}[label=\textbf{B\arabic*}]\setlength{\itemsep}{-1mm}
		\item $\norm{\Delta \eb_k}  < \frac{\lambda_{\min}^+(\Hb)}{\lambda_{\max}(\Hb)} \frac{c_k}{\norm{\wb_0}} $. \label{B1}
		\item $ y_0 < \xb_0^T \wb_P^* + \frac{1-\xb_0^T \Hb^{-1}\xb_0}{\norm{\xb_0}} \min\limits_{1\le k\le d} \bigg(
		(\frac{c_k^*}{c_k}-1)\frac{1}{\lambda_k} - \frac{\norm{\wb_{\tilde{P}}^*-\wb_0}}{\norm{\wb_0}}
		\frac{1}{\lambda_{\max}(\Hb)} \bigg) c_k \lambda_k^2 $.	\label{B2}
		\item $\tilde{\Hb}$ is full rank. \label{B3}
		\item $\xb_0^T\wb_P^* < 0$. \label{B4}
	\end{enumerate}
	\vspace{-6pt}
	For the gradient flow initialized with infinitesimally small norm, suppose above assumptions hold on every boundary of partitions that gradient flow visits until it converges. Then, $\norm{\wb(t)}$ strictly increases until it converges.
\end{theorem}

\begin{proof}[Proof of Theorem~\ref{thm: norm increasing}] \allowdisplaybreaks
	Before we start, recall Proposition \ref{prop: hyperrectangle}.
	To prove the norm-increasing property, it is enough to show that the gradient flow is contained in the associated hyperrectangle of each partition. More precisely, since a gradient flow never escapes  the associate hyperrectangle if it was initialized in the hyperrectangle in each partition, we only need to consider the boundary of partitions. See Figure \ref{fig: norm increasing theorem}.
	
	We claim that the gradient flow is always contained in the hyperrectangle of each partition, and prove it by induction. First, we initialize a gradient flow with infinitesimally small norm and contained in the hyperrectangle.
	By induction hypothesis, suppose $\wb_0$ is contained in the hyperrectangle of $P$ and assumptions \ref{B1}, \ref{B2} hold.
	We want to show that $\wb_0$ is in the hyperrectangle of $\tilde{P}$, the next partition $\wb(t)$ goes. In other words, $0 < \tilde{c}_k < \tilde{c}_k^*$ for all $k=1,2,\cdots,d$. 
	The proof consists of two parts. Note that $\tilde{\Hb}=\Hb\pm\xb_0\xb_0^T$ depend on whether $\xb_0$ is activated or deactivated. 
	\begin{enumerate}[itemindent=14mm, leftmargin=0mm]
		\item[{\bf Claim 1.}] $\tilde{c_k}>0$ : 
		
		For the $k$-th eigenvector $\wb_k$, 
		\begin{align*}
			\tilde{\lambda}_k \tilde{c}_k - \lambda_k c_k
			&= \wb_0 ^T (\tilde{\lambda}_k \tilde{\eb}_k - \lambda_k \eb_k) \\
			&= \wb_0^T( \tilde{\Hb} \tilde{\eb}_k - \Hb \eb_k ) \\
			&= \wb_0^T((\Hb \pm\xb_0\xb_0^T)(\eb_k + \Delta \eb_k) - \Hb \eb_k) \\
			&= \wb_0^T( \Hb\Delta \eb_k \pm \xb_0^T\tilde{\eb}_k \xb_0  ) \\
			&= \wb_0^T \Hb\Delta \eb_k \pm (\xb_0^T \tilde{\eb_k}) (\wb_0^T\xb_0) \\
			&= \wb_0^T \Hb \Delta \eb_k
		\end{align*}
		since $\wb_0^T\xb_0=0$ ($\wb_0$ is on the activation boundary). However, from \ref{B1},
		\begin{align*}
			|\wb_0^T \Hb \Delta \eb_k| &\le \norm{\wb_0} \cdot \norm{\Hb}_2 \cdot \norm{\Delta \eb_k} \\
			&< \norm{\wb_0} \cdot \norm{\Hb}_2 \cdot 
			\frac{\lambda_{\min}^+(\Hb)}{\lambda_{\max}(\Hb)} \cdot
			\frac{c_k}{\norm{\wb_0}} \\
			&= \lambda_{\min}^+(\Hb) c_k \\
			&\le \lambda_k c_k .
		\end{align*}
		Therefore,
		$ \tilde{\lambda}_k \tilde{c}_k = \lambda_k c_k + \wb_0^T \Hb\Delta \eb_k >0$ and we conclude $\tilde{c}_k > 0$.
		
		\item[{\bf Claim 2.}] $\tilde{c}_k < \tilde{c}_k^*$:
		\begin{align*}
			\tilde{c}_k^* - \tilde{c}_k 
			&= (\wb_{\tilde{P}}^* - \wb_0) ^T \tilde{\eb}_k \\
			&= (\wb_P^* +\Delta \wb^* - \wb_0) ^T (\eb_k + \Delta \eb_k) \\
			&= (\wb_P^* - \wb_0 )^T\eb_k + (\wb_P^* + \Delta \wb^*-\wb_0)^T \Delta \eb_k + \Delta \wb^{*T}\eb_k \\
			&= (c_k^* - c_k) + (\wb_{\tilde{P}}^* - \wb_0)^T \Delta \eb_k + \Delta \wb^{*T}\eb_k.
		\end{align*}
		For the last two terms of the right hand side,
		\begin{align*}
			| (\wb_{\tilde{P}}^* &- \wb_0)^T \Delta \eb_k + \Delta \wb^{*T}\eb_k | \\
			&< | \Delta \wb^{*T}\eb_k | + | (\wb_{\tilde{P}}^* - \wb_0)^T \Delta \eb_k | \\
			&= | (\wb_{\tilde{P}}^*-\wb_P^*)^T\eb_k | + | (\wb_{\tilde{P}}^* - \wb_0)^T \Delta \eb_k | \\
			&= | ((\Hb\pm\xb_0\xb_0^T)^{-1}(\qb-y_0\xb_0)-\Hb^{-1}\qb)^T\eb_k | + | (\wb_{\tilde{P}}^* - \wb_0)^T \Delta \eb_k | \\
			&= \left| \left( (\Hb^{-1} \mp \frac{\Hb^{-1}\xb_0\xb_0^T \Hb^{-1}}{1\pm \xb_0^T \Hb^{-1}\xb_0})(\qb-y_0\xb_0) -\Hb^{-1}\qb \right)^T\eb_k \right| + | (\wb_{\tilde{P}}^* - \wb_0)^T \Delta \eb_k | \\
			&= \left| \frac{\mp \xb_0^T\wb_P^* - y_0}{1 \pm\xb_0^T \Hb^{-1}\xb_0}  \eb_k^TH^{-1}\xb_0  \right| + | (\wb_{\tilde{P}}^* - \wb_0)^T \Delta \eb_k | \\
			&= \left| \frac{\mp \xb_0^T \wb_P^* -y_0}{1 \pm \xb_0^T \Hb^{-1}\xb_0} \right| \cdot |\frac{1}{\lambda_k} \eb_k^T\xb_0 | + | (\wb_{\tilde{P}}^* - \wb_0)^T \Delta \eb_k | \\
			&< \frac{|y_0 \pm \xb_0^T\wb_P^*|}{1-\xb_0^T \Hb^{-1}\xb_0} \cdot \frac{\norm{\xb_0}}{\lambda_k} + \norm{\wb_{\tilde{P}}^* - \wb_0 } \cdot \norm{\Delta \eb_k} \\
			&< \frac{y_0 - \xb_0^T\wb^*}{1-\xb_0^T \Hb^{-1}\xb_0} \cdot \frac{\norm{\xb_0}}{\lambda_k} + \norm{\wb_{\tilde{P}}^* - \wb_0 } \cdot \frac{\lambda_k}{\lambda_{\max}(\Hb)} \frac{c_k}{\norm{\wb_0}} \\
			&< c_k^* - c_k .
		\end{align*}
		where the second to last inequality holds by triangle inequality, \ref{B1} and \ref{B4}. The last inequality holds by \ref{B2} :
		\begin{align*}
			0 &\le \frac{y_0 - \xb_0^T\wb_P^*}{1-\xb_0^T \Hb^{-1}\xb_0} \\
			&< \frac{1}{\norm{\xb_0}}\bigg(
			(\frac{c_k^*}{c_k}-1)\frac{1}{\lambda_k} - \frac{\norm{\wb_{\tilde{P}}^*-\wb_0}}{\norm{\wb_0}}
			\frac{1}{\lambda_{\max}(\Hb)}
			\!\bigg) c_k \lambda_k^2.
		\end{align*}
		
		Therefore,
		$$ \tilde{c}_k^* - \tilde{c}_k = (c_k^* - c_k) + (\wb_{\tilde{P}}^* - \wb_0)^T \Delta \eb_k + \Delta \wb^{*T}\eb_k >0.
		$$
		Finally, Claim 1 and 2 with induction hypothesis prove $\wb_0$ is in the hyperrectangle of $\tilde{P}$. Therefore, the gradient flow always contained in hyperrectangles and thus $\norm{\wb}$ increases until it converges.
	\end{enumerate}
	\vspace{-6mm}
\end{proof}

\subsection[]{Special case : $d=2$}
\begin{figure}[ht!]
	\centering
	\includegraphics[width=0.6\linewidth]{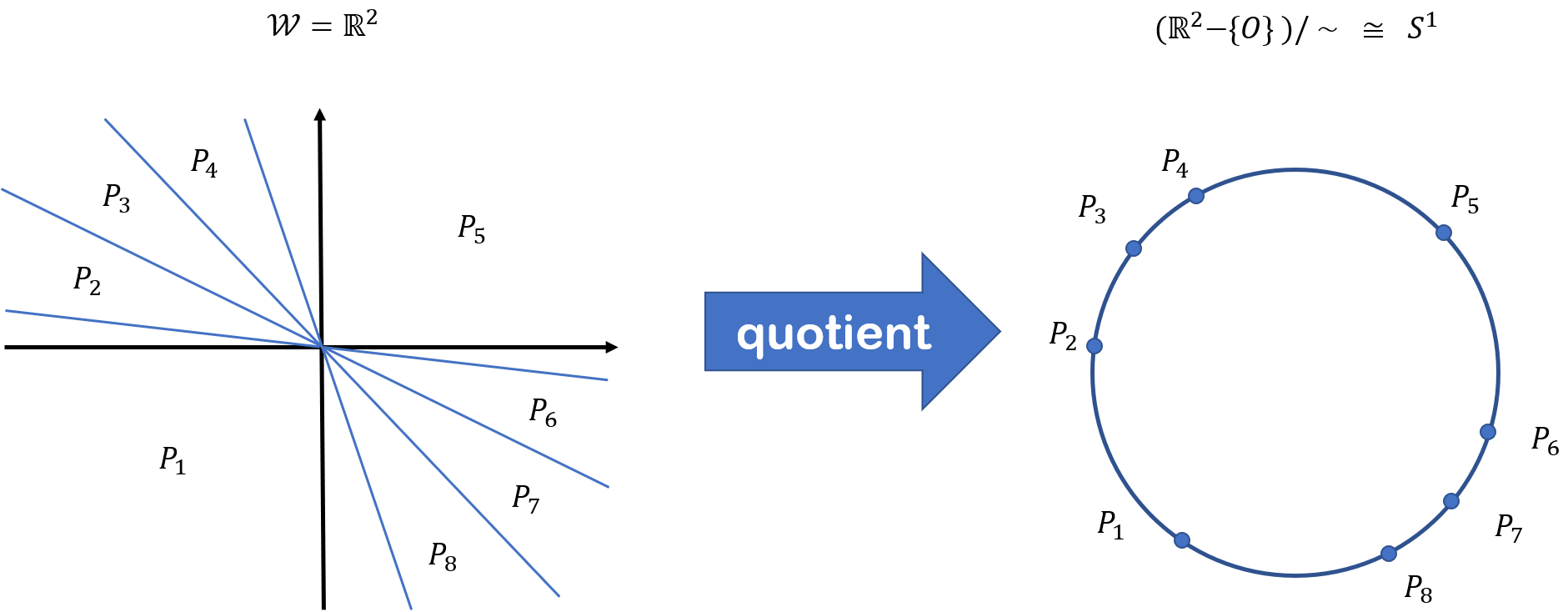}
	\caption{Topology of the partitioned parameter space in $\Rd^2$. By the positive homogeneity of the parameter space, its quotient space is homeomorphic to $S^1$, which has a cyclic order.}
	\label{fig: quotient topology}
\end{figure}
\begin{proof}[Proof of Lemma~\ref{lem:ordering}]
	Since $d=2$, each partition has exactly two activation boundaries. 
	Now recall the positive homogeneity of partitions from Proposition \ref{prop: basic properties}. Define an equivalent relation in the parameter space, defined by $\wb_1 \sim \wb_2$ if and only if $\wb_1=\alpha\wb_2$ for some $\alpha>0$. Now omit the origin point from $\Wc=\Rd^2$, and take quotient by this equivalent relation. Then, we get 
	$$ (\Rd^2 - \{O\} )/\sim \quad \cong \quad S^1 .
	$$
	Therefore, we can impose an cyclic `order' for partitions. For example, Figure \ref{fig: quotient topology} shows 8 partitions ordered by clockwise direction.
	
	Now consider any two partitions in 2nd quadrant. since we can impose an order from 3rd quadrant $\rightarrow$ 2nd quadrant $\rightarrow$ 1st quadrant by clockwise direction, one partition is obtained by activating (or deactivating) some data from another partition. Similar argument holds for any two partitions in 4th quadrant, which completes the proof.
\end{proof}

Still we need some lemmas to prove the global convergence. 
Now we refer the norm-increasing subset. Proposition \ref{prop: hyperrectangle} changes to the following lemma.

\begin{lemma}[Norm-increasing subset] \label{prop: norm-increasing ReLU}
	Consider a gradient flow \eqref{eq: ReLU gradient flow} under \ref{asm: input} and \ref{asm: label}. Define $g(\wb):= -\frac{d}{dt} \frac12 \norm{\wb}^2 = \wb^T(\Hb(\wb)\wb-\qb(\wb))$.
	Then, for a gradient flow $\wb(t)$,
	\begin{enumerate}\setlength{\itemsep}{-1mm}
		\item $g(\wb)<0 $ if and only if $\norm{\wb(t)}$ increases.
		\item $g(\wb)$ is continuous.
		\item Any local minimum $\wb^*$ is on the boundary of the norm-increasing subset, i.e. $g(\wb^*)=0$.
	\end{enumerate}
\end{lemma}
\begin{proof}
	\begin{enumerate}
		\item Since $g(\wb)=-\frac{d}{dt} \frac12 \norm{\wb}^2$, $\norm{\wb}$ increases if and only if $g(\wb)<0$.
		
		\item Now we show that the function $g(\wb)$ is continuous. Indeed, it is continuous on each linearly partitioned region from the definition. Now consider a boundary of partitions and let $\xb_0$ be the data which determines the boundary. Let $P$ and $\tilde{P}$ be the two partitions such that 
		$\{\xb \in \Xb ~|~ \xb \sim P\} = \{\xb \in \Xb ~|~ \xb \sim \tilde{P} \} \cup \{\xb_0\}$. See Figure \ref{fig: norm increasing theorem}. Now distinguish $g(\wb)$ in each region by
		\begin{align*}
			g_P(\wb) &= \wb^T (\sum_{\xb_i \sim P}\xb_i \xb_i^T) \wb - \sum_{\xb_i \sim P} y_i\xb_i^T\wb, \\
			g_{\tilde{P}}(\wb) &= \wb^T (\sum_{\xb_i \sim \tilde{P}}\xb_i \xb_i^T) \wb - \sum_{\xb_i \sim \tilde{P}} y_i\xb_i^T\wb \\
			&= \wb^T (\sum_{\xb_i \sim P}\xb_i \xb_i^T) \wb - \wb^T\xb_0 \xb_0^T\wb - \sum_{\xb_i \sim P} y_i\xb_i^T\wb + y_0\xb_0^T\wb \\
			&= g_P(\wb) - \wb^T\xb_0 \xb_0^T\wb + y_0\xb_0^T\wb.
		\end{align*}
		Then, for $\wb_0$ on the activation boundary of $\xb_0$, since $\xb_0^T\wb_0=0$, we get $g_P(\wb_0)= g_{\tilde{P}}(\wb_0)$. Therefore, $g(\wb)$ is continuous in $\Wc$.
		
		\item Let $\wb^*$ be a local minimum. Then $\wb^T\dot{\wb}=0$ at $\wb^*$, since it is a stationary point. Thus $g(\wb^*)=0$.
	\end{enumerate}
	\vspace{-6mm}
\end{proof}

Now we define {\em all-activated partition} as the partition which activates all data. The following lemma guarantees the global convergence, if all-activated region contains a local minimum.

\begin{lemma} \label{lem: convergence to all activated partition}
	Consider a single-neuron ReLU network under \ref{asm: input}, \ref{asm: label}, and \ref{asm: underparameterization} with $d=2$. {Suppose there exists stationary point $\wb^*$ in all-activated partition. At initialization $\wb_0$, further suppose that $\qb$ defined in \eqref{eq: Hq} has positive entries.}
	Then there exists $\delta>0$ such that if $\norm{\wb_0}<\delta$, any gradient flow initialized at $\wb_0$ converges to the global minimum $\wb^*$, without any deactivation of data.
\end{lemma}
\begin{proof}
    {
    By Lemma \ref{lem: all activated}, we can assume that the initialization point $\wb_0$ is in the all-activated partition.}
	For a data $\xb_i$, define $h_i(t):=\wb(t)^T\xb_i$. By the above claim, $\wb_0$ is in all-activated partition thus $h_i(0)>0$ for all $i=1,2,\dots,n$. Now we want to show that  $h_i(t) \ge 0$ for all $t \ge 0$. 
	For the Hessian matrix $\Hb := \sum\limits_{i=1}^n \xb_i\xb_i^T$, let $\{(\lambda_k, \eb_k)\}_{k=1}^2$ be its eigenvalues and eigenvectors with $\lambda_1 > \lambda_2 > 0$. Considering the hyperrectangle of $\Hb$ in Definition \ref{def: hyperrectangle}, $c_k^*:=\wb^{*T}\eb_k >0$ by definition. From \eqref{eq: linear gradient flow solution eigenvector version}, we get
	\begin{align*}
	    \wb(t) &= (1-e^{-\lambda_1 t}) \eb_1 \eb_1^T \wb^* + (1-e^{-\lambda_2 t})\eb_2 \eb_2^T \wb^* \notag \\
	    &= (1-e^{-\lambda_1 t}) c_1^*\eb_1 + (1-e^{-\lambda_2 t})c_2^*\eb_2.
	\end{align*}
	
	Let $c_k := \xb_i^T \eb_k$. Then,
	\begin{align} \label{eq: zzz}
	    h_i(t) &= [(1-e^{-\lambda_1 t})c_1^*\eb_1 + (1-e^{-\lambda_2 t})c_2^*\eb_2]^T\xb_i \notag \\
	    &= (1-e^{-\lambda_1 t})c_1c_1^* + (1-e^{-\lambda_2 t})c_2c_2^* .
	\end{align}
	From the eigenvector decomposition, from \ref{asm: input}, we know $c_1>0$. However, we are not sure for the sign of $c_2$. Thus we consider both two cases, separately. Note that $c_1c_1^* + c_2c_2^* = \xb_i^T\sum\limits_{k=1}^2 \eb_k\eb_k^T \wb^* = \xb_i^T\wb^* >0.$
	
	If $c_2 \ge 0$, then all terms in \eqref{eq: zzz} are nonnegative, thus $h_i(t) \ge 0$ for all $t \ge 0$. Otherwise, if $c_2 < 0$, deforms \eqref{eq: zzz} to
	\begin{align*}
	    h_i(t) &= c_1c_1^* + c_2c_2^* - e^{-\lambda_1 t}c_1c_1^* - e^{-\lambda_2 t}c_2c_2^* \\
	    &= c_1c_1^* + c_2c_2^* - e^{-\lambda_1 t}c_1c_1^*
	    - e^{-\lambda_1 t}c_2c_2^* + e^{-\lambda_1 t}c_2c_2^*- e^{-\lambda_2 t}c_2c_2^* \\
	    &= (c_1c_1^* + c_2c_2^*)(1-e^{-\lambda_1 t}) + c_2c_2^* (e^{-\lambda_1 t} - e^{-\lambda_2 t}) \\
	    &= \xb_i^T\wb^*(1-e^{-\lambda_1 t}) + c_2c_2^* (e^{-\lambda_1 t} - e^{-\lambda_2 t}).
	\end{align*}
	Then all terms are nonnegative, since $c_2 (e^{-\lambda_1 t} - e^{-\lambda_2 t})>0$. Therefore, whether $c_2 \ge 0$ or not, we conclude $h_i(t) \ge 0$ for all $t \ge 0$. Since $\xb_i$ could be any data, we proved that no data is deactivated on the gradient flow, {which implies gradient flow is always in the all-activated partition}. Finally, by Theorem \ref{thm: more support vectors}, the convergent point is the global minimum.
\end{proof}

Now we show norm-increasing property for $d=2$.

\begin{figure}[ht!]
	\centering
	\begin{subfigure}[b]{0.3\textwidth}
        \centering
        \includegraphics[width=\textwidth]{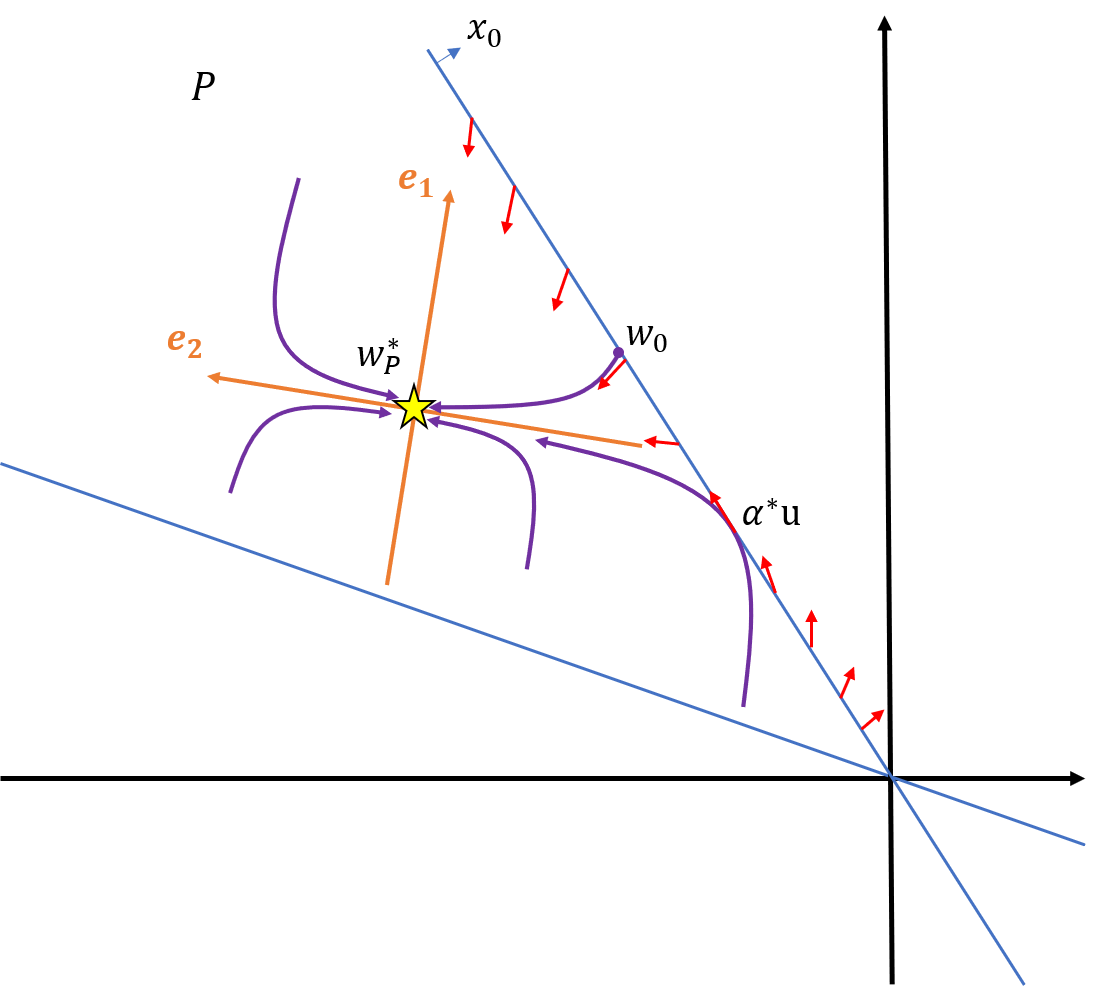}
        \caption{}
    \end{subfigure}
	\hfill
	\begin{subfigure}[b]{0.3\textwidth}
        \centering
        \includegraphics[width=\textwidth]{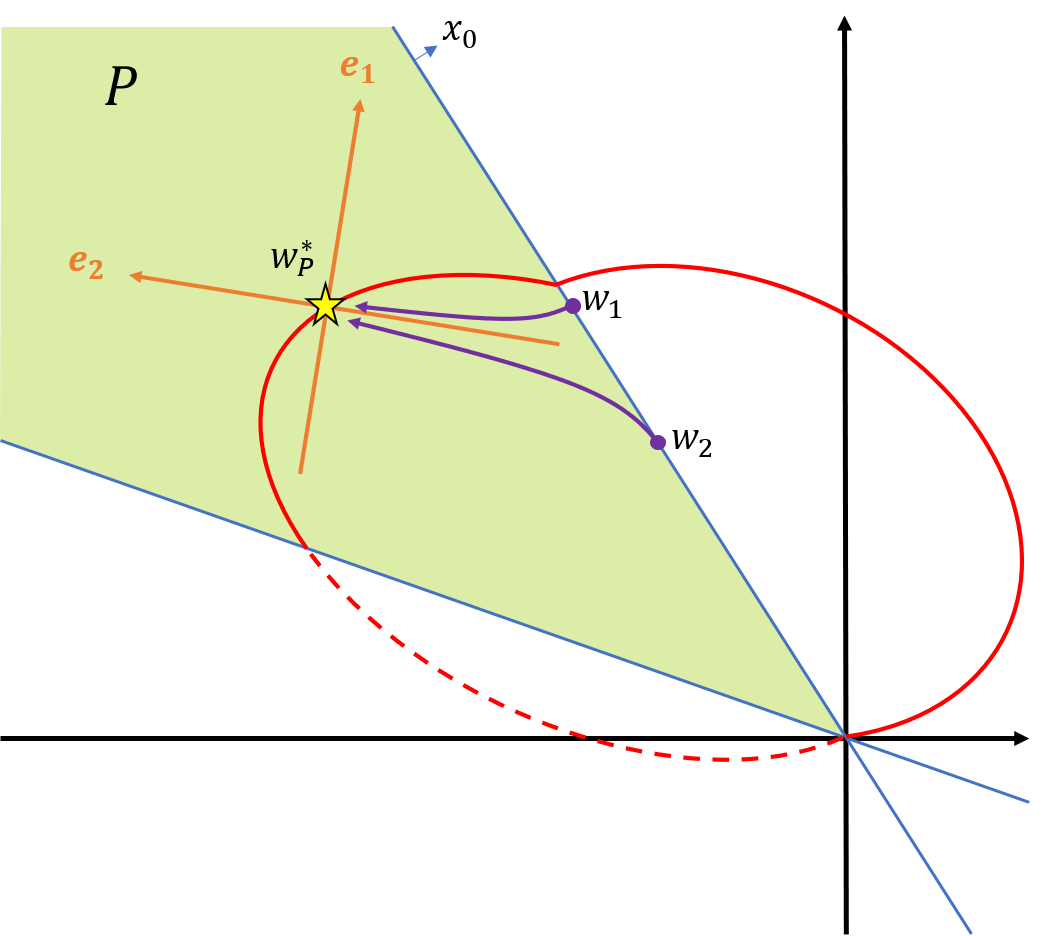}
        \caption{}
    \end{subfigure}
	\hfill
	\begin{subfigure}[b]{0.3\textwidth}
        \centering
        \includegraphics[width=\textwidth]{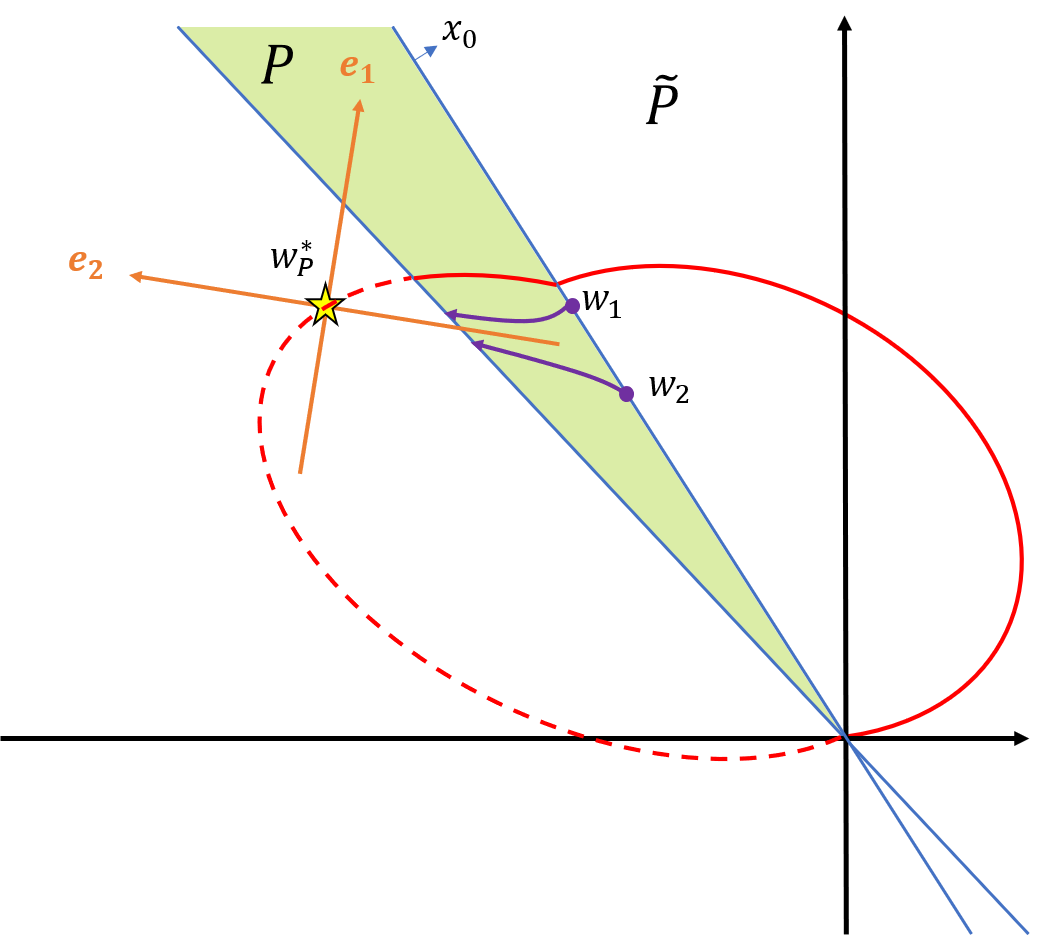}
        \caption{}
    \end{subfigure}
	\caption{Gradient flow dynamics in $\Rd^2$. 
	    In (a), the eigenvectors of $\Hb$ in $P$ are drawn. Since $\lambda_1 > \lambda_2$, $\eb_2$ is the asymptotic line of convergence. Some examples of gradient flows are drawn by the purple curves.
	    In (b), it is the case that the virtual minimum of $P$ is contained in $P$. Since the asymptotic line through the $\wb_P^*$, it converges in the norm increasing subset. 
	    In (c), it shows the alternative case $\wb_P^* \not\in P$. Note that the gradient flow is still in the norm increasing subset during escaping $P$.
	}
	\label{fig: no revisit}
\end{figure}

\begin{lemma}[Norm increasing for $d=2$] \label{lem: norm-increasing 2D}
	Consider a single-neuron ReLU network under \ref{asm: input}, \ref{asm: label}, and \ref{asm: underparameterization}. Recall the function $g(\wb)$ defined in Proposition~\ref{prop: norm-increasing ReLU}. 
	{For a gradient flow $\wb(t)$, suppose its initialization $\wb_0$ is in the all-activated partition and has sufficiently small norm. 
	Then, $g(\wb(t))<0$ for all $t>0$.}
	Furthermore, let $P^*$ be the partition that $\wb(t)$ converges in. Then $P^*$ is the first partition which contains a stationary point, that the gradient flow has {ever met}. 
\end{lemma}
\begin{proof}
	Since $g(\wb(t))$ is continuous on boundary of partitions(Proposition \ref{prop: norm-increasing ReLU}), we only need to consider interior of each partition. 
	{
	Now we propose two claims : \\
	1. $g(\wb(t))<0$ for all $t>0$. \\
	2. $\wb(t)$ cannot exit a partition that contains a stationary point. More precisely, if $\wb(t)$ enters to a partition $P$ which contains its virtual minimizer $\wb_P^* = \argmin\limits_{\wb \in \Wc} L_P(\wb)$ in $P$, then $\wb(t)$ must converge to $\wb_P^*$.
	}
	
	{We prove these claims by induction on order of partitions.} In the first partition, $g(\wb)<0$ since the norm of gradient flow is increasing. For the second statement, if the all-activated partition contains a stationary point, then the gradient flow converges to this minimum without any deactivation by Lemma \ref{lem: convergence to all activated partition}. Moreover, {$\norm{\wb(t)}$ strictly increases until it converges, by Theorem \ref{thm: norm increasing linear}}.
	
	Now, we suppose $g(\wb)<0$ in a partition $\tilde{P}$, and consider the next partition $P$. Let $\xb_0$ be the data which determines the common boundary of $\tilde{P}$ and $P$ (See Figure \ref{fig: no revisit}). Let $\wb_0$ be the meeting point of gradient flow and the activation boundary of $\xb_0$. {We know $g(\wb_0) \le 0$ from induction hypothesis.} Let $\eb_1$ and $\eb_2$ be the eigenvectors of $\Hb(\wb)$, such that $\lambda_1 > \lambda_2$. Note that a line through $\wb_P^*$ with direction $\eb_2$ is the asymptotic line of gradient flows \cite{zill2020advanced}. Now, we have the following two cases.
	\begin{enumerate}
		\item $\wb_P^*$ is in $P$.\\
		See Figure \ref{fig: no revisit}(b). 
		$\wb_P^*$ is in $P$ and $g(\wb_P^*)=0$ by Proposition \ref{prop: norm-increasing ReLU}. Note that the asymptotic line through $\wb_P^*$ and parallel to $\eb_2$.
		If $\wb_0$ is inside the hyperrectangle of $P$, then $\wb(t)$ converges to $\wb_P^*$ inside the hyperrectangle by Theorem \ref{thm: norm increasing linear}, thus $g(\wb(t)) \le 0$. Otherwise, if $\wb_0$ is not in the hyperrectangle, then convexity of the gradient flow and ellipse $g(\wb)=0$ is opposite as shown in Figure \ref{fig: no revisit}(b).
		Therefore, the gradient flow converges to $\wb_P^*$ keeping $g(\wb) \le 0$. This concludes that $g(\wb)$ is negative until it converges, and the gradient flow converges in the first partition that contains its local minimum.
		
		\item $\wb_P^*$ is not in $P$.\\
		By the induction hypothesis, the gradient flow has never visited a partition that contains its minimum itself. Therefore, $\wb_P^*$ is under the partition $P$ where the order of partitions is defined by Lemma \ref{lem:ordering} (See Figure \ref{fig: no revisit}(c)). 
		Now we apply the same argument about convexity above. If $\wb(t)$ is in the hyperrectangle of $P$, then it escapes $P$ within the hyperrectangle, thus $g(\wb)<0$. If $\wb(t)$ is not in the hyperrectangle, then convexity of the gradient flow and ellipse is opposite as shown in Figure \ref{fig: no revisit}(c), thus the gradient flow escapes $P$ keeping $g(\wb)<0$. Therefore, for all cases, $g(\wb)<0$ until the gradient flow escapes $P$.
	\end{enumerate}
	By induction, we showed $g(\wb)<0$ in every partition, which completes the proof.
\end{proof}

\begin{proof}[Proof of Lemma \ref{lem: no revisit}]
	Consider a deactivation of the gradient flow. Suppose the gradient flow $\wb(t)$ enters to a partition $P$ by deactivating $\xb_0$ at $\wb_0$. Let $\ub:=\frac{\wb_0}{\norm{\wb_0}}$, $\Hb:= \sum\limits_{\xb_i \sim P} \xb_i\xb_i^T$, and $\qb:= \sum\limits_{\xb_i \sim P} y_i\xb_i$. 
	From Proposition \ref{prop: effect of norm}, we know 
	$\xb_0$ is deactivated if and only if $\alpha > \alpha^*$ (See Figure  \ref{fig: no revisit} (a)).
	
	However, by Lemma \ref{lem: norm-increasing 2D}, we know that norm of the gradient flow $\norm{\wb(t)}$ strictly increases. Therefore, the norm has been greater than the point where $\xb_0$ was deactivated.
	It means, there is no re-activation of $\xb_0$. In other words, a gradient flow initialized with an infinitesimally small norm never revisits a partition that already been traversed. 
\end{proof}


\begin{proof}[Proof of Theorem~\ref{thm: 2D input}]
	By Lemma \ref{lem: all activated}, the gradient flow enters to the all-activated partition. From there, by No revisit lemma (Lemma \ref{lem: no revisit}), there is no re-activation. In other words, the gradient flow only deactivates data until it converges. Moreover, by Lemma \ref{lem: norm-increasing 2D}, it converges to a minimum in the first partition such that it contains.
	Therefore, it has maximum number of support vectors, which is the global minimum by Theorem \ref{thm: more support vectors}.
\end{proof}

However, we show that this global convergence does not hold in high dimension ($d>2$) in general, in Example \ref{exmp: deactivation}. The key point is that No revisit lemma (Lemma \ref{lem: no revisit}) does not hold for $d>2$. The following lemma and theorem show that for single-neuron linear networks with $d>2$, a gradient flow may reactivate at most $(d-1)$ times.

\begin{lemma}\label{lem: number of zeros of exponential}
	For some constants $a_i, b_i$, and $c$, the equation $\sum\limits_{i=1}^n a_i e^{b_it}=c$ has at most $n$ zeros.
\end{lemma}
\begin{proof}
	We prove a weaken version of this Lemma first, which is $c=0$ case. \\
	{\bf Claim} : $\sum\limits_{i=1}^n a_i e^{b_it}=0$ has at most $n-1$ solution.\\
	{\it Proof of Claim} : We use mathematical induction on $n$. For $n=1$, equation $a_1e^{b_1t}=0$ has no solution and for $n=2$, it has at most one solution $t=\frac{1}{b_1-b_2}\log\frac{-a_2}{a_1}$ if $\frac{-a_2}{a_1}>0$. Now suppose it holds for $n$ and consider $n+1$ step. WLOG assume $a_{n+1} \neq 0$ and deform the equation $ \sum\limits_{i=1}^{n+1} a_ie^{b_it}=0 $ to
	\begin{align} \label{eq: exponential equation}
		\sum\limits_{i=1}^{n+1} a_ie^{b_it} = a_{n+1}e^{b_{n+1}t}(1+\sum\limits_{i=1}^{n} \frac{a_i}{a_{n+1}} e^{(b_i-b_{n+1})t})=0.
	\end{align}
	Now let $g(t) = 1+\sum\limits_{i=1}^{n} \frac{a_i}{a_{n+1}} e^{(b_i-b_{n+1})t}$. Then, $$g'(t)=\sum\limits_{i=1}^{n} \frac{a_i(b_i-b_{n+1})}{a_{n+1}} e^{(b_i-b_{n+1})t}=0$$
	has at most $n-1$ solutions, by the induction. Therefore, \eqref{eq: exponential equation} has at most $n$ solutions, by Rolle's Theorem. This completes the proof of the claim. $\hfill \qedsymbol$
	
	Now we return to prove the original lemma. For given $\sum\limits_{i=1}^n a_i e^{b_it}=c$, it has a solution if and only if $\sum\limits_{i=1}^n a_i e^{(b_i+b_{n+1})t}-ce^{b_{n+1}t}=0$, which has at most $n+1$ zeros by the claim. 
	
	Finally, we provide one example that $\sum\limits_{i=1}^n a_i e^{b_it}=c$ has exactly $n$ zeros. Take $c = 1$, then there exists $a_i$'s such that $t=0,1,2,\cdots,n-1$ are zeros for some distinct $b_i$'s because 	
	$$ 
	\begin{pmatrix}
		1 & 1& \cdots & 1 \\
		e^{b_1} & e^{b_2} & \cdots & e^{b_n} \\
		\vdots & \vdots & \ddots & \vdots \\
		e^{(n-1)b_1} & e^{(n-1)b_2} & \cdots & e^{(n-1)b_n}
	\end{pmatrix} \begin{pmatrix} a_1 \\ a_2 \\ \vdots \\ a_n	\end{pmatrix} 
	= \begin{pmatrix} 1 \\ 1 \\ \vdots \\ 1 \end{pmatrix}
	$$
	above matrix is called a Vandermonde matrix, which is known to be invertible \cite{anton2003contemporary}.
\end{proof}

\begin{theorem}[Reactivation of linear networks for $d>2$]
	Consider single-neuron linear networks under \ref{asm: input}, \ref{asm: label}, and \ref{asm: underparameterization}. There exists dataset $\{(\xb_i, y_i)\}_{i=1}^n$ such that for a hyperplane $V$, a gradient flow crosses $V$ at most $d$ times.
\end{theorem}
\begin{proof}
		Recall \eqref{eq: linear gradient flow solution eigenvector version}. By translation and rotation, WLOG we can assume $\wb^*$ to be origin and $\Hb= \diag(e^{\lambda_1}, \cdots, e^{\lambda_d})$. Therefore, the gradient flow is given by
		$$ \wb(t) = \diag(e^{\lambda_1t}, \cdots, e^{\lambda_dt})\wb_0 = \begin{bmatrix}
			    e^{\lambda_1t} w_{01} \\ e^{\lambda_2t} w_{02} \\
			    \vdots \\ e^{\lambda_dt} w_{0d}
			\end{bmatrix}  $$
		Now suppose that we have a hyperplane $V$ determined by $\{\wb ~|~ \wb^T\vb = c\}$. Then the number of intersection points of the gradient flow and this hyperplane is given by the number of zeros of the equation
		\begin{align} \label{eq: exponential}
			\wb(t)^T\vb-c = \sum_{k=1}^d e^{\lambda_kt} w_{0k}v_i -c = 0
		\end{align}
		Finally, by Lemma \ref{lem: number of zeros of exponential}, equation (\ref{eq: exponential}) has at most $d$ zeros. Therefore, a gradient flow of single-neuron linear networks can across the hyperplane $V$ at most $d$ times. 
	\end{proof}

We provide Example \ref{exmp: reactivation} in Section \ref{sec: experiments} which shows an reactivation in a single-neuron ReLU network, for $d=3$ case.

\section{Detail settings of the experiments.} \label{app: experiments}
We use PyTorch library with GPU Gigabyte GeForce GTX 1080 Ti to implement all experiments.
The gradient flow is implemented by a gradient descent with  small learning rate.

\subsection{Detail settings of Example \ref{exmp: initialization with infinitesimally small norm}}
This toy example is a single-neuron ReLU network with $n=5$ and $d=2$ setting. The training dataset used in this example is
\begin{align*}
    [ \xb_1 ~\xb_2&~\xb_3~\xb_4~\xb_5 ] = \\
    &\begin{bmatrix}
        0.8858 & 0.4338 & 0.6739 & 0.0221 & 0.2322 \\
        0.0244 & 0.8852 & 0.0399 & 0.4778 & 0.8717
    \end{bmatrix}, \\
    [ y_1~y_2&~y_3~y_4~y_5] =
    [ 0.6111 ~ 0.9397 ~ 1.8694 ~ 2.7104 ~ 1.3089 ].
\end{align*}

Three blue gradient flows in Figure~\ref{fig: initialized at zero is important} are initialized at 
$$ \wb_0 = \begin{bmatrix} 0.0001 \\ 0.0001 \end{bmatrix}, \begin{bmatrix} 0\\ 8 \end{bmatrix}, \text{ and }
\begin{bmatrix} 0 \\ 45 \end{bmatrix}.$$

Seven initialization points near origin are given below. Note that they satisfy $\norm{\wb_0}<0.3$, which is depicted by the red dashed circle in Figure \ref{fig: initialized at zero is important}.

$$ \wb_0 = \begin{bmatrix} -0.05\\ 0.15 \end{bmatrix},
\begin{bmatrix} 0.1\\ -0.1 \end{bmatrix},
\begin{bmatrix} -0.15\\ 0.15 \end{bmatrix},
\begin{bmatrix} -0.25\\ 0.02 \end{bmatrix},
\begin{bmatrix} 0.01\\ -0.1 \end{bmatrix},
\begin{bmatrix} 0.1\\ -0.2 \end{bmatrix},
\text{ and }
\begin{bmatrix} 0.17 \\ 0.1 \end{bmatrix}.$$

 Learning rate is set to 0.005, and total number of iterations is set to 200K.

\subsection{Detail settings of Example \ref{exmp: deactivation}} \label{subsec: exmp deactivation}
This toy example considers single-neuron linear and ReLU networks with $n=d=3$. The training dataset used in this example is
\begin{align*}
    [ \xb_1 ~\xb_2 ~\xb_3 ] &= \begin{bmatrix}
		1 & 1 & 2 \\
		0 & 2 & 0 \\
		2 & 0 & 0
	\end{bmatrix},\\
	[ y_1 ~ y_2 ~ y_3 ] &= [ 0.05 ~ 6 ~ 0.5 ].
\end{align*}
 As mentioned in Section \ref{sec: experiments}, we define $h_i(\wb):=\wb^T\xb_i$ for each data $\xb_i$. 
We used the initial point sampled from 
$$ \wb_0 = 0.0001 \times U([0,1]^3), $$
where $U([0,1])$ means uniform distribution in $[0,1]$.
Both linear and ReLU networks share the initial point, where learning rate and the total number of iterations are set to 0.005 and 10,000, respectively. 

The value of $h_i(\wb)$ functions are plotted in Figure~\ref{fig: deactivation}. Note that $\xb_1$ is deactivated on the gradient flow of the ReLU network ($h_1(\wb)<0$), and gradient flows coincide until that time. Since there is no reactivation again, the gradient flow of the ReLU network converges to a local minimum(See Figure~\ref{fig: deactivation}(b)).

\subsection{Detail settings of Example \ref{exmp: reactivation}} \label{subsec: exmp reactivation}
This toy example considers single-neuron linear and ReLU networks with $n=4$ and $d=3$. The training dataset used in this example is
\begin{align*}
    [ \xb_1 ~\xb_2 ~\xb_3~\xb_4 ] &= \begin{bmatrix}
		1 & 1 & 1 & 0 \\
		0 & 2 & 0 & 1 \\
		1 & 1 & 2 & 0
	\end{bmatrix},\\
	[ y_1 ~ y_2 ~ y_3 ~y_4 ] &= [ 0.1 ~~ 0.2 ~~ 4 ~~ 0.1 ].
\end{align*}
 As mentioned in Section \ref{sec: experiments}, we define $h_i(\wb):=\wb^T\xb_i$ for each data $\xb_i$. 
We used the initial point sampled from 
$$ \wb_0 = 0.0001 \times U([0,1])^3, $$
and both linear and ReLU networks share the initial point.
For both networks, learning rate is set to 0.005 and the total number of iterations is set to 20K. 

$h_i(\wb(t))$ is plotted in Figure~\ref{fig: reactivation}. Note that $\xb_4$ is deactivated on the gradient flow of ReLU network ($h_4(\wb)<0$) at the beginning, and reactivated after. Since it converges in the all-activated partition, gradient flows of linear and ReLU networks converge to the same point, which is the global minimum(See Figure~\ref{fig: reactivation}(b)).

\begin{figure*}[ht!]
	\centering
	\begin{subfigure}[b]{0.4\textwidth}
        \centering
        \includegraphics[width=\textwidth]{deactivation_convergence.png}
        \caption{}
    \end{subfigure}
	\hfill
	\begin{subfigure}[b]{0.5\textwidth}
        \centering
        \includegraphics[width=\textwidth]{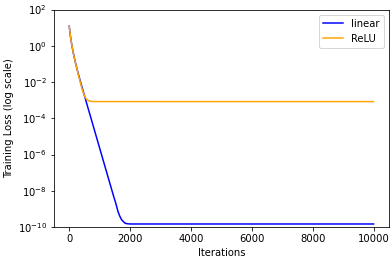}
        \caption{}
    \end{subfigure}
	\\
	\begin{subfigure}[b]{0.3\textwidth}
        \centering
        \includegraphics[width=\textwidth]{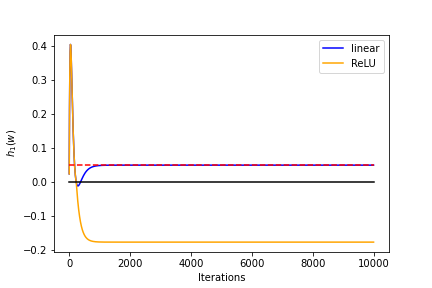}
        \caption{}
    \end{subfigure}
	\hfill
	\begin{subfigure}[b]{0.3\textwidth}
        \centering
        \includegraphics[width=\textwidth]{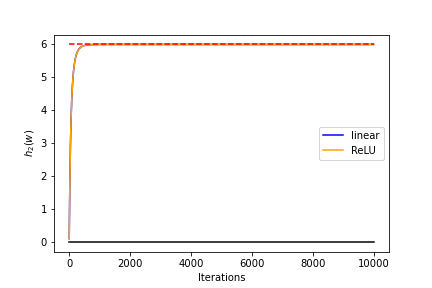}
        \caption{}
    \end{subfigure}
	\hfill
	\begin{subfigure}[b]{0.3\textwidth}
        \centering
        \includegraphics[width=\textwidth]{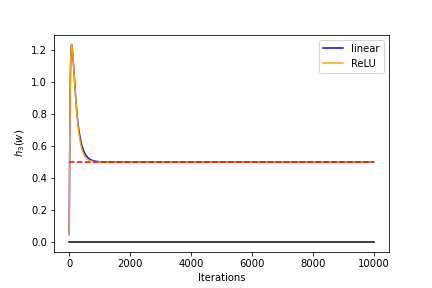}
        \caption{}
    \end{subfigure}
	\caption{Deactivation in ReLU network during training single-neuron linear and ReLU networks on the same dataset.
		In (a), gradient flows of linear and ReLU networks are described by blue and orange curves.
		In (b), training loss curves of two gradient flows are plotted.
		In (c), (d) and (e), graphs of $h_i(\wb)=\wb^T\xb_i$ for linear and ReLU networks are plotted. Asymptotic lines are drawn by red dashed lines.
		Observe that $\xb_1$ is deactivated at near $500$ iterations, and the ReLU network never reactivates it($h_1(\wb_r^*)<0$) while the linear network reactivates it again ($h_1(\wb_l^*)>0$).}
	\label{fig: deactivation}
\end{figure*}
\begin{figure*}[ht!]
	\centering
	\begin{subfigure}[b]{0.4\textwidth}
        \centering
        \includegraphics[width=\textwidth]{revisit_convergence.png}
        \caption{}
    \end{subfigure}
	\hfill
	\begin{subfigure}[b]{0.5\textwidth}
        \centering
        \includegraphics[width=\textwidth]{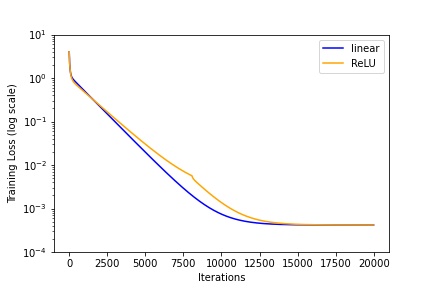}
        \caption{}
    \end{subfigure}
	\\
	\begin{subfigure}[b]{0.24\textwidth}
        \centering
        \includegraphics[width=\textwidth]{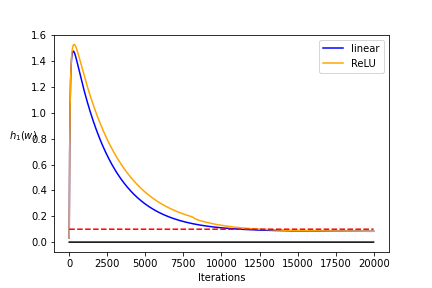}
        \caption{}
    \end{subfigure}
	\hfill
	\begin{subfigure}[b]{0.24\textwidth}
        \centering
        \includegraphics[width=\textwidth]{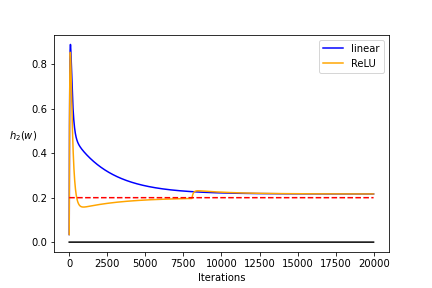}
        \caption{}
    \end{subfigure}
	\hfill
	\begin{subfigure}[b]{0.24\textwidth}
        \centering
        \includegraphics[width=\textwidth]{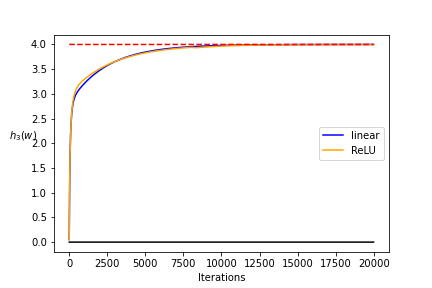}
        \caption{}
    \end{subfigure}
	\hfill
	\begin{subfigure}[b]{0.24\textwidth}
        \centering
        \includegraphics[width=\textwidth]{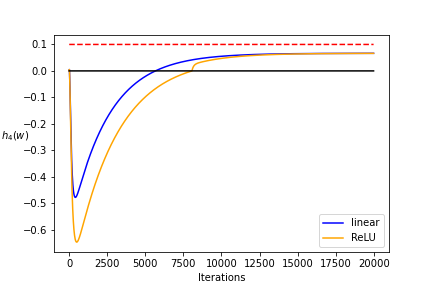}
        \caption{}
    \end{subfigure}
	\caption{Reactivation in ReLU networks.
		In (a), gradient flows of linear and ReLU networks are described by blue and orange curves.
		In (b), training loss curves of two gradient flows are plotted.
		In (c), (d), (e), and (f), graphs $h_i(\wb)=\wb^T\xb_i$ of linear and ReLU networks are plotted. Asymptotic lines $h_i(\wb)=y_i$ are drawn by red dashed lines. Observe that $\xb_4$ is deactivated short time after initialized, but the ReLU network reactivates it at near $8500$ iterations.
	}
	\label{fig: reactivation}
\end{figure*}

\end{document}